\documentclass{article}

\PassOptionsToPackage{numbers, sort&compress}{natbib}
\usepackage[main, final]{neurips_2026}

\usepackage[utf8]{inputenc} 
\usepackage[T1]{fontenc}    
\usepackage{hyperref}       
\usepackage{url}            
\usepackage{booktabs}       
\usepackage{multirow}
\usepackage{amsfonts}       
\usepackage{nicefrac}       
\usepackage{microtype}      
\usepackage{xcolor}         
\usepackage{tikz}           
\usetikzlibrary{fit,calc,arrows.meta,patterns}

\usepackage{xparse}             
\usepackage{amsthm, aliascnt}   
\usepackage{thmtools}           
\usepackage{amsmath, amssymb}   
\usepackage{mathtools}          
\usepackage{cleveref}
\usepackage{dsfont}             
\usepackage[inline]{enumitem}
\usepackage{algorithm}          
\usepackage[noend]{algpseudocode}
\usepackage{tcolorbox}
\usepackage{placeins}           
\usepackage{titletoc}           

\declaretheoremstyle[
  style=definition,
  notefont=\bfseries,
  notebraces={(}{)}
]{bolddefinition}

\declaretheorem[name=Theorem,numberwithin=section]{theorem}
\declaretheorem[name=Lemma,sibling=theorem]{lemma}

\declaretheorem[name=Proposition,sibling=theorem]{proposition}

\declaretheorem[style=bolddefinition,name=Definition,sibling=theorem]{definition}

\crefname{theorem}{theorem}{theorems}
\Crefname{theorem}{Theorem}{Theorems}
\crefname{lemma}{lemma}{lemmas}
\Crefname{lemma}{Lemma}{Lemmas}
\crefname{corollary}{corollary}{corollaries}
\Crefname{corollary}{Corollary}{Corollaries}
\crefname{proposition}{proposition}{propositions}
\Crefname{proposition}{Proposition}{Propositions}
\crefname{problem}{problem statement}{problem statements}
\Crefname{problem}{Problem Statement}{Problem Statements}
\crefname{assumption}{assumption}{assumptions}
\Crefname{assumption}{Assumption}{Assumptions}
\crefname{example}{example}{examples}
\Crefname{example}{Example}{Examples}
\crefname{definition}{definition}{definitions}
\Crefname{definition}{Definition}{Definitions}
\crefname{criteria}{design criteria}{design criteria}
\Crefname{criteria}{Design criteria}{Design criteria}
\crefname{remark}{remark}{remarks}
\Crefname{remark}{Remark}{Remarks}

\newcounter{boxedproblem}
\renewcommand{\theboxedproblem}{\Roman{boxedproblem}}
\crefalias{boxedproblem}{problem}

\algdef{SE}[FORALLPAR]{ParForAll}{EndParForAll}[1]
  {\algorithmicforall\ #1\ \textbf{in parallel}\ \algorithmicdo}
  {\algorithmicend\ \algorithmicforall}

\algtext*{EndParForAll}

\newcommand{\fin}{\mathrm{fin}}

\newcommand{\R}{\mathbb{R}}
\newcommand{\N}{\mathbb{N}}

\newcommand{\valspace}{\ensuremath{\mathcal{V}}}
\newcommand{\keyspace}{\ensuremath{\mathcal{K}}}
\newcommand{\quespace}{\ensuremath{\mathcal{Q}}}
\newcommand{\kvspace}{\ensuremath{\mathcal{X}}}
\newcommand{\Prob}{\ensuremath{\mathcal{P}}}
\newcommand{\hilbert}{\ensuremath{\mathcal{H}}}
\newcommand{\ball}{\ensuremath{\mathsf{B}}}

\NewDocumentCommand{\kvmeasures}{s o}{%
  \ensuremath{%
    \Prob%
    \IfValueTF{#2}
      {_{#2}}
      {_{\mathrm{fin}}}%
    \IfBooleanT{#1}{^\star}%
    (\kvspace)%
  }%
}
\newcommand{\quemeasures}{\Prob(\quespace)}
\NewDocumentCommand{\compspace}{m o}{%
  \ensuremath{%
    \mathcal A_{#1}%
    \IfValueTF{#2}
      {^{\mathrm{#2}}}%
      {^{\mathrm{aw}}}%
  }%
}

\DeclareMathOperator{\Att}{Att}
\NewDocumentCommand{\att}{s m m}{%
  \ensuremath{%
    \Att%
    \IfBooleanTF{#1}
      {\bigl(#2 \,|\, #3\bigr)}
      {(#2 \mid #3)}%
  }%
}
\newcommand{\diff}{\mathop{}\!\mathrm{d}}

\newcommand{\Err}{\mathcal{E}}
\NewDocumentCommand{\err}{s m m m}{%
  \ensuremath{%
    \Err_{#2,#3}%
    \IfBooleanTF{#1}
      {\bigl(#4\bigr)}
      {(#4)}%
  }%
}

\newcommand{\cov}{\ensuremath{\Sigma}}
\newcommand{\Cov}{\mathrm{Cov}}

\newcommand{\minimax}[2]{%
  \ensuremath{\mathfrak{R}\!\left(#1,#2\right)}%
}

\DeclareMathOperator{\Ex}{\mathbb{E}}

\NewDocumentCommand{\dotp}{m m o}{%
  \ensuremath{%
    \left\langle#1,#2\right\rangle
    \IfValueTF{#3}
      {_{\mathrm{#3}}}%
      {}%
  }%
}

\DeclareMathOperator{\rank}{rank}
\DeclareMathOperator{\ran}{ran}
\DeclareMathOperator{\im}{im}
\DeclareMathOperator{\supp}{supp}

\DeclareMathOperator{\trace}{tr}
\DeclareMathOperator{\vspan}{span}
\DeclareMathOperator{\tail}{tail}

\DeclareMathOperator{\bigO}{\mathcal{O}}

\DeclareMathOperator{\reduce}{\mathsf{REDUCE}}

\title{The risk of KV cache compression}

\author{%
    Lukas Haverbeck\thanks{Corresponding author. Work conducted at Stanford University.} \\
    RWTH Aachen University\\
    \texttt{lukas.haverbeck@rwth-aachen.de} \\
    \And
    Carmen Amo Alonso\\
    Stanford University\\
    \texttt{camoalon@stanford.edu} \\
    \And
    Andres Felipe Posada-Moreno\\
    RWTH Aachen University\\
    \texttt{andres.posada@dsme.rwth-aachen.de} \\
    \And
    Sebastian Trimpe\\
    RWTH Aachen University\\
    \texttt{trimpe@dsme.rwth-aachen.de} \\
    \And
    Marco Pavone\\
    Stanford University\\
    \texttt{pavone@stanford.edu} \\
}

\begin{document}

\maketitle

\begin{abstract}
    Transformer inference on long sequences is expensive because softmax attention repeatedly reads from a large KV cache.
    The prevalent approach to this bottleneck is \emph{KV cache compression}, which replaces the full cache with a compact summary.
    Despite its practical importance, the design of such summaries is largely driven by empirical experimentation.
    On the theoretical side, existing results show that KV cache compression can be impossible in the worst case, but offer little systematic guidance for designing algorithms in regimes where accurate compression is possible.
    We bridge this gap by characterizing the \emph{minimax risk of KV cache compression} in terms of the intrinsic compressibility of a cache, revealing when and how accurate compression is possible.
    In particular, we show that, once attention can sharply retrieve individual tokens, compression can only uniformly outperform random sampling by exploiting information about future queries, regardless how redundant the cache is inherently.
    Additionally, our results yield novel design principles for KV cache compression under causal masking that map efficiently to prefill and autoregressive decoding while achieving minimax-optimal risk.
    We instantiate these principles in a practical algorithm and illustrate how our theoretical findings can be put to practical use.
    Overall, our results opens a principled avenue for KV cache compression with theoretical guarantees.
\end{abstract}

\section{Introduction}
Transformers have become the dominant architecture for processing and generating long sequences across various domains \citep{devlin2019bert,baevski2020wav2vec,rives2021biological,avsec2021effective} due to their excellent ability to model complex interactions between tokens. During inference, these interactions are mediated through the \emph{key--value (KV) cache} storing all previous tokens, and each new token is computed from this cache via softmax attention \citep{vaswani2017attention}. The KV cache therefore serves as the model's working memory of the past. While this persistent access to past data makes Transformers powerful, it is also costly. As the sequence grows, the cache grows with it, and each new step must read from an increasingly large memory. For long-context applications, this creates a major bottleneck in both memory use and runtime.

A promising approach to overcoming this bottleneck is \emph{KV cache compression}: replacing the original cache with a much smaller collection of key--value pairs that approximately preserves its attention outputs.\footnote{The term also refers to numerical compression of the stored tensors. We focus on reducing the number of key--value pairs.} The challenge is deciding which information to retain, particularly because the importance of a cached token depends on the queries that will attend to it in the future. A token that receives little attention from one query may be essential to another, while multiple tokens may have sufficiently similar effects on attention that retaining all of them is unnecessary. Consequently, a cache that can be compressed accurately for one set of queries may be difficult to compress for another. Since future queries are generally not known at compression time, this raises a fundamental question: \emph{How accurately a KV cache without knowing which information will be needed in the future?}

A growing body of work has developed practical methods for KV compression, using heuristic proxies for future relevance such as recency, accumulated attention scores, and attention sinks to decide which tokens to retain \citep{liu2023scissorhands,zhang2023h2o,xiao2023efficient,li2024snapkv,cai2024pyramidkv,adnan2024keyformer}. These methods demonstrate that substantial compression is often possible in practice. However, their empirical success does not explain what makes a KV cache compressible in the first place, how accurately it can be compressed for a given budget, or how close existing methods come to the best achievable accuracy. Meanwhile, theoretical results establish worst-case limitations for KV compression \citep{haris2025compression,haverbeck2026efficiently} and identify structural conditions under which accurate compression is possible \citep{kochetkova2025streaming,carrell2025low,schroder2026wildcat,zandieh2024subgen}. What remains missing is a general characterization of compression accuracy that accounts for the redundancy in how future queries interact with the cached information through softmax attention directly, without relying on an external sufficient condition for exploitable redundancy.

In this work, we show that, once attention is sufficiently expressive to retrieve individual tokens, the redundancy of a KV cache depends inseparably on what is known about future queries. When compression must preserve attention without information about which queries will occur, even sophisticated compression methods generally cannot improve upon the approximation guarantees of plain random sampling. In contrast, when information about the distribution of future queries is available, substantially more accurate compression is possible. We establish this separation by studying the \emph{minimax risk of KV compression}: the smallest approximation error achievable for a given compression budget under the distribution of queries accessing the cache after compression. To this end, we pose KV cache compression as a problem of compressing a discrete probability measure and use this connection to show that a query-agnostic compressor cannot outperform the rate of simple Monte Carlo sampling, while a query-aware compressor can attain the error of Monte Carlo sampling with oracle-optimal variance reduction, corresponding to much lower error in the attention approximation. Beyond quantitative rates, our results reveal algorithmic design principles for KV cache compression that extend to causally masked attention, covering both prefill and autoregressive decoding. These principles enable systematic improvements over the Monte Carlo baseline when partial information or a heuristic model of the future-query distribution is available, approaching the query-aware minimax-optimal risk as the available information becomes more accurate.

In short, our contributions are as follows:
\textbf{(1)} We recast KV compression as sparse approximation of a measure, strictly generalizing token eviction, and formalize its minimax risk.
\textbf{(2)} We prove tight upper and lower bounds on this risk, resulting in a sharp separation between algorithms with and without access to the future-query distribution.
\textbf{(3)} We derive design principles for efficient KV compression during causally masked prefill and autoregressive decoding, enabling improvements over the query-agnostic minimax risk when even partial information about the query distribution is available.
\textbf{(4)} We instantiate these principles in a concrete KV compression algorithm and illustrate its viability through targeted experiments on LongBench.
Taken together, our results open an avenue for practical KV compression with theoretical guarantees.
\section{Related work}

A large body of recent work reduces KV cache size by selecting tokens according to empirical proxies for future relevance. These proxies include recency and attention sinks \citep{xiao2023efficient,zhang2023h2o}, accumulated or persistent attention mass \citep{liu2023scissorhands,zhang2023h2o,adnan2024keyformer}, and attention patterns observed during prefill or across layers \citep{li2024snapkv,cai2024pyramidkv}. More recent empirical work explicitly approximates how the cache will be queried in the future to decide which tokens to retain \citep{galim2026draft,wang2025lookahead,ahn2026lookahead,yang2026attentionpredictor,dong2026foresight,devoto2025expected,tang2025adaptive,tang2024quest}, suggesting that query-dependent selection rules can improve performance over algorithms depending only on the cache itself. Overall, the strong empirical performance of these methods suggests that real Transformer caches often contain substantial redundancy. Our work complements these empirical findings by studying the best achievable compression error theoretically, and our findings support a separation between query-aware and query-agnostic methods. Previous theoretical work identifies impossibility results for KV compression, and attention approximation more broadly, when attention can sharply retrieve individual tokens and for adversarial inputs \citep{alman2023fast,gupta2025subquadratic,haris2025compression,haverbeck2026efficiently}, whereas remarkably strong approximation guarantees can be guaranteed uniformly over all inputs \citep{kochetkova2025streaming,carrell2025low,schroder2026wildcat,gong2026express} when the sharpness of attention is limited or under more structural assumptions on the KV cache \citep{zandieh2024subgen,aliakbarpour2026support}. These results identify important limitations and tractable regimes, but do not take into account redundancy inherent to the interaction between queries and the KV cache through softmax attention directly. Our upper bounds require no a priori structural model of the data beyond bounded values, which is weaker than limiting attention sharpness. Finally, we heavily draw on the connection between KV compression and approximation of probability measures, which has been exploited before to apply distribution compression techniques to KV compression \citep{carrell2025low,schroder2026wildcat,gong2026express}.
\section{Problem setup} \label{sec:setting}

We study KV compression: the problem of replacing the context seen by a single softmax attention head with a compact summary while preserving its outputs on future queries, in order to reduce memory use and computational cost.
Since softmax attention depends on the context only through the induced weighting of its key--value pairs, we formalize KV compression as sparse approximation of a probability measure over tokens.
Our goal is to characterize the achievable error for this approximation problem as a function of the context's \emph{intrinsic compressibility} under a query distribution, which we formalize as a minimax risk in terms of a data-dependent complexity measure.

\subsection{KV compression as sparse approximation of a probability measure}
In their seminal work, \citet{vaswani2017attention} define softmax attention as
\begin{equation} \label{eq:attention-standard}
    \Att\bigl(q \mid k_1, \dots, k_n, v_1, \dots, v_n\bigr)
    =
    \frac{
        \sum_{i=1}^n \exp\left(\left\langle q, k_i \right\rangle/\sqrt{d_k}\right) \, v_i
    }{
        \sum_{i=1}^n \exp\left(\left\langle q, k_i \right\rangle/\sqrt{d_k}\right)
    },
\end{equation}
where \(q\in\quespace\) is a \emph{query}, \(k_i\in\keyspace\) are \emph{keys}, and \(v_i \in\valspace\) are \emph{values}, in their respective spaces \(\keyspace,\quespace \subseteq \R^{d_k}\) and \(\valspace \subseteq \R^{d_v}\).
We write \(\kvspace \coloneqq \keyspace \times \valspace\) for the key--value space, and assume throughout that values are bounded in norm by a constant \(V > 0\), which is usually not restrictive in practice.

The collection \((k_i,v_i)_{i=1}^n\) of key--value pairs in \eqref{eq:attention-standard} is called the \emph{KV cache}. 
We study the problem of \emph{KV compression}: finding a compact representation of this collection that approximately preserves the attention map in \eqref{eq:attention-standard}.
Since \eqref{eq:attention-standard} depends on the KV cache only through a weighted combination, we represent the KV cache as a finitely supported probability measure \(P\) and define
\begin{equation} \label{eq:attention}
    \att*{q}{P}
    \coloneqq
    \int_{\kvspace} a_k(q \mid P)\, v \, \diff P(k,v),
    \qquad
    a_k(q \mid P)
    \coloneqq
    \frac{\kappa(q,k)}{\int_{\kvspace} \kappa(q,k') \, \diff P(k',v')}
\end{equation}
for a Gaussian kernel \(\kappa(q,k) \coloneqq \exp(-\|q-k\|_2^2 / (2\sqrt{d_k}))\) on \(\R^{d_k}\) with scale parameter \(\sqrt{d_k}\). 
This definition coincides with \eqref{eq:attention-standard} whenever \(P \propto \sum_{i=1}^n \exp\bigl(\|k_i\|^2 / (2\sqrt{d_k})\bigr) \cdot \delta_{(k_i,v_i)}\). For the remainder of this paper, we work with attention as defined in \eqref{eq:attention} so that attention depends on the KV cache only through the \emph{context measure} \(P\).

To formalize KV compression from this viewpoint, we introduce the notation \(\Prob(\mathcal S)\) for the probability measures on a subset \(\mathcal S \subseteq \hilbert\) of a Hilbert space \(\hilbert\), and denote by \(\Prob_\fin(\mathcal S)\) those with finite support and by \(\Prob_K(\mathcal S)\) those supported on at most \(K\) points.
KV compression then becomes the following problem.

\refstepcounter{boxedproblem}
\label{prb:kv-compression}
\begin{tcolorbox}[
  title={Problem statement~\theboxedproblem\ (KV compression)},
  colback=white,
  colframe=black,
  colbacktitle=black, 
  coltitle=white,    
  fonttitle=\bfseries
]
Given a budget \(K \ge 1\), a context measure \(P \in \kvmeasures\), and a query distribution \(\nu \in \quemeasures\), find a compressed context measure \(\hat P \in \kvmeasures[K]\) minimizing the mean squared error
\[
    \textstyle
    \err*{P}{\nu}{\hat P}
    \coloneqq
    \Ex_{q \sim \nu}\bigl\|\att*{q}{P} - \att*{q}{\hat P}\bigr\|_2^2.
\]
\end{tcolorbox}

\Cref{prb:kv-compression} strictly generalizes token eviction \citep{liu2023scissorhands,zhang2023h2o,xiao2023efficient,li2024snapkv,cai2024pyramidkv}, which seeks a small, unweighted subset of \(K\) tokens that approximately preserve the attention map.
In our setting, these tokens form the support of \(\hat P\) but we allow arbitrary sparse reweightings of the cache.

\subsection{Capturing the intrinsic compressibility of a context}

To understand when accurate compression is possible, we identify how changes to \(P\) affect softmax attention.
Since \(P\) and \(\hat P\) are both probability measures, their difference \(\Delta \coloneqq \hat P - P\) is a signed, zero-mass measure.
Compression can therefore be viewed as a reweighting of the original context, and its error is determined by how this reweighting perturbs the softmax numerator and normalizer.

Indeed, a direct computation shows that, for every query \(q \in \quespace\),
\begin{equation} \label{eq:attention-reweighting}
    \att{q}{\hat P}-\att{q}{P}
    =
    \frac{
        \int_\kvspace a_k(q \mid P)\bigl(v-\att{q}{P}\bigr)\,\diff \Delta(k,v)
    }{
        1+\int_\kvspace \bigl(a_k(q \mid P)-1\bigr)\,\diff \Delta(k,v)
    }.
\end{equation}
Thus, every compression perturbation enters attention through two linear functionals of \(\Delta\): a centered numerator term and a centered renormalization term.
This identifies the pair of query-indexed coefficients with which each \((k,v) \in \kvspace\) contributes to the compression-induced reweighting of tokens.
We collect these coefficients with the following definition.

\begin{definition}[Response profile] \label{def:response-profile}
    Let \(P \in \kvmeasures\) be a context measure and \(\nu \in \quemeasures\) a query distribution. We call
    \begin{equation} \label{eq:response-profile}
        \Gamma_P(k,v) : \quespace \to \R^{d_v}\oplus\R,
        \qquad
        q \mapsto
        \Bigl(
            a_k(q \mid P)\bigl(v-\att{q}{P}\bigr),
            \;
            V\bigl(a_k(q \mid P)-1\bigr)
        \Bigr)
    \end{equation}
    the \emph{response profile} of \(P\). We view \(\Gamma_P(k,v)\) as an element of the Hilbert space \(\hilbert_\nu \coloneqq L_2(\nu, \R^{d_v} \oplus \R)\) and define the covariance operator
    \[
        \textstyle
        \cov_{P,\nu} \coloneqq \int_\kvspace \Gamma_P(k,v) \otimes \Gamma_P(k,v) \,\diff P(k,v).
    \]
\end{definition}

The first coordinate of \(\Gamma_P(k,v)\) records the token's contribution to the centered numerator perturbation in \eqref{eq:attention-reweighting}.
The second coordinate records its contribution to the centered denominator perturbation, scaled by the value bound \(V\) so that both coordinates have the same units.
The query distribution enters through the norm on \(\hilbert_\nu\): Directions that matter on frequent queries are weighted heavily, while directions that matter only on rare queries are weighted lightly.
Consequently, \(\cov_{P,\nu}\) measures the amount and geometry of query-visible variation in the context.
While we do not claim that \(\cov_{P,\nu}\) is the \emph{only} object to capture the intrinsic redundancy of a context, we argue that it is a meaningful choice.
To substantiate this, we give three examples illustrating that the spectral decay of \(\cov_{P,\nu}\) tracks the intuitive difficulty of compression in several natural regimes.

\begin{restatable}[Needles in a haystack]{example}{expneedles}
    Suppose a context \(P\) consists of \(m\) approximately equal ``hay tokens'' and \(n \ll m\) ``needle tokens'' that are well separated from all other tokens and dominate under \(\nu\). Then, \(\cov_{P,\nu}\) has few large eigenvalues.
\end{restatable}
\begin{restatable}[Token clusters]{example}{expclusters}
    Suppose a context \(P\) consists of \(m\) well-separated clusters of nearly interchangeable tokens, and the \(j\)-th cluster dominates under \(\nu\) with frequency of order \(j^{-(1+\alpha)}\). Then, the eigenvalues of \ \(\cov_{P,\nu}\) decay as \(r^{-\alpha}\).
\end{restatable}
\begin{restatable}[Random lookup]{example}{explookup}
    Suppose a context \(P\) consists of \(m\) pairwise well-separated tokens, and the query distribution \(\nu\) is diffuse. Then, \(\cov_{P,\nu}\) has many large eigenvalues.
\end{restatable}

\subsection{Characterizing the achievable error of KV compression algorithms}

Our goal in this paper is to characterize the fundamental limits of KV compression.
Specifically, we seek minimax rates in terms of the intrinsic compressibility of a context by asking: \emph{What error does the best compression algorithm incur on its worst input \((P,\nu)\), as a function of \(\cov_{P,\nu}\)?}
To make this precise, we first specify what information a compression algorithm may use.

\begin{definition}[Compressor]
    A \(K\)-compressor is a map \(A: \kvmeasures \times \quemeasures \times \Omega \to \kvmeasures[K]\) over a probability space \((\Omega,\mathcal F,\mathbb P)\) of random seeds. We write \(\compspace{K}\) for the class of all such compressors.
\end{definition}

This definition allows compression algorithms to be randomized.
It also allows the compressor to depend on the query distribution \(\nu\).
We call this the \emph{query-aware} setting.
In many applications, however, the compressed summary must be formed before the future query distribution is known, which we call the \emph{query-agnostic} setting, which we capture through the following restricted class.

\begin{definition}[Query-agnostic compressor] \label{def:agnostic-compressor}
    We call a \(K\)-compressor \(A \in \compspace{K}\) query-agnostic if \(A(P,\nu,\omega) = A(P,\rho,\omega)\) for all \(\nu,\rho \in \quemeasures\), \(P \in \kvmeasures\), and \(\omega \in \Omega\). We write \(\compspace{K}[ag]\) for the class of all query-agnostic \(K\)-compressors.
\end{definition}

We measure the performance of a class of compressors by its worst-case expected mean squared error over a prescribed family of contexts and query distributions.

\begin{definition}[Minimax risk]
    Fix \(K \ge 1\).
    For a family of compressors \(\mathcal C \subseteq \compspace{K}\), and a family of context and query measures \(\mathcal F \subseteq \kvmeasures \times \quemeasures\), define the minimax compression risk
    \[
        \minimax{\mathcal C}{\mathcal F}
        \coloneqq
        \inf_{A \in \mathcal C}
        \sup_{(P,\nu) \in \mathcal F}
        \Ex_{\omega \sim \mathbb P}\!\left[
            \err*{P}{\nu}{A(P, \nu, \omega)}
        \right].
    \]
\end{definition}

Our objective is to characterize \(\minimax{\compspace{K}}{\mathcal F}\) and \(\minimax{\compspace{K}[ag]}{\mathcal F}\) through the response covariances \(\cov_{P,\nu}\) realized by \((P,\nu)\in\mathcal F\), and to use this characterization to identify what compressed summaries must preserve in order to attain minimax-optimal risk.
\section{Minimax rates for KV compression} \label{sec:minimax-rates}

We now study the minimax risk of KV compression.
\Cref{fig:approach} illustrates our approach: compressing a KV cache amounts to replacing the context measure \(P\) by a sparse reweighting, so that many tokens receive weight \(0\) while the remaining tokens are reweighted to preserve the effect of the full context on the attention map.
To track this effect, we represent each key--value pair \((k,v)\in\kvspace\) by its response profile \(\Gamma_P(k,v)\in\hilbert_\nu\), which records the contribution of \((k,v)\) to the attention numerator and denominator under reweightings of \(P\).
These individual contributions combine through their \emph{barycenter} in \(\hilbert_\nu\).
As we show, reweightings causing small barycenter displacement also cause small attention error, recasting KV compression as a balancing problem in Hilbert space: given a context measure \(P\), reweight \(P\) so that its support is limited to at most \(K\) tokens while the barycenter of the response profiles remains nearly fixed.
\Cref{prb:barycenter-sparsification} formulates this problem abstractly.

\begin{figure}[tbp]
    \centering
    \input{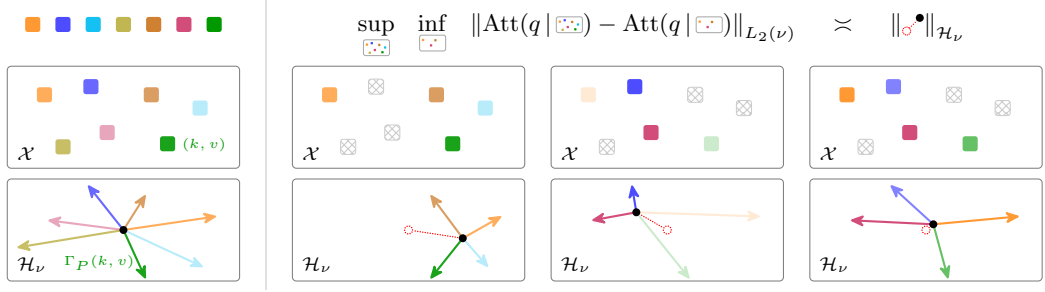}
    \caption{
        Illustration of our approach.
        \textbf{Left:} We view the token sequence as a measure \(P\) over key--value pairs \((k,v) \in \kvspace\), which we map to their response profiles \(\Gamma_P(k,v) \in \hilbert_\nu\).
        \textbf{Right:} Compressing \(P\) amounts to a sparse reweighting of tokens, which moves the barycenter of the response profiles.
        The size of this displacement depends on the retained tokens and their assigned weights, and controls the attention error incurred by compression.
        KV compression therefore reduces to a balancing problem in Hilbert space: find a \(K\)-atomic reweighting that keeps the barycenter in \(\hilbert_\nu\) nearly fixed.
    }
    \label{fig:approach}
\end{figure}

\refstepcounter{boxedproblem}
\label{prb:barycenter-sparsification}
\begin{tcolorbox}[
  title={Problem statement~\theboxedproblem\ (Sparse balancing in Hilbert space)},
  colback=white,
  colframe=black,
  colbacktitle=black, 
  coltitle=white,    
  fonttitle=\bfseries
]
Given \(K \ge 1\), a Hilbert space \(\hilbert\), and \(\mu \in \Prob_\fin(\hilbert)\), find \(\hat \mu \in \Prob_K(\supp(\mu))\) minimizing
\[
    \textstyle
    \widetilde\Err_{\mu}(\hat \mu)
    \coloneqq
    \bigl\|\int_{\hilbert} g \,\diff\hat\mu(g)\bigr\|_{\hilbert}^2.
\]
\end{tcolorbox}

The rest of the section turns this reformulation into tight minimax bounds.
\Cref{sec:reduction} proves the transfer back to KV compression: a \(K\)-atomic reweighting whose squared barycenter displacement is at most \(\varepsilon\) incurs attention error at most \(C\varepsilon\), for a universal constant \(C>0\).
\Cref{sec:upper-bounds} then proves two balancing bounds for \Cref{prb:barycenter-sparsification} and transfers them to the query-aware and query-agnostic minimax compression risk.
Finally, \Cref{sec:lower-bounds} proves that these bounds are tight for sufficiently rich classes of contexts and query distributions.
For that, we work in the regime where softmax attention can sharply retrieve individual tokens, formalized by the following assumption.

\begin{restatable}{assumption}{asprichness} \label{asp:richness}
Fix \(K \ge 1\) and let \(\ball_d(r)\) denote the radius-\(r\) ball in \(\R^d\).
Assume that
\begin{enumerate*}[label=\textup{(\roman*)}, itemjoin={{; }}, itemjoin*={{; and }}]
    \item \label{itm:richness:1} \(\keyspace \cap \quespace \supseteq \ball_{d_k}\bigl(8 \, d_k^{1/4} \sqrt{\log K}\bigr)\)
    \item \label{itm:richness:2} \(d_k,d_v \ge 10 \log K\)
    \item \label{itm:richness:3} \(\valspace = \ball_{d_v}(V)\)
\end{enumerate*}.
\end{restatable}

Condition~\ref{itm:richness:1} gives the minimal norm growth needed for sharp attention among \(K\) tokens: resolving one token against \(K\) competitors requires logit gaps of order \(\log K\), while radius-\(B\) keys and queries give gaps at most \(2B^2/\sqrt{d_k}\), forcing \(B\gtrsim d_k^{1/4}\sqrt{\log K}\).
Condition~\ref{itm:richness:2} allows mild dimension growth with context length, and Condition~\ref{itm:richness:3} ensures that the value norm bound \(V\) is at the right scale.

\paragraph{Notation.}
For a Hilbert space \(\hilbert\) and a positive trace-class operator \(L\) on \(\hilbert\), with eigenvalues \(\lambda_1(L)\ge\lambda_2(L)\ge\cdots\ge0\), multiplicities counted, set \(\tail_r(L)\coloneqq\sum_{j>r}\lambda_j(L)\).
For \(x,y\ge0\), write \(x\lesssim y\) if \(x\le Cy\) and \(x\gtrsim y\) if \(x\ge cy\) for universal \(C,c>0\), and write \(x\asymp y\) if both hold.

All proofs and constructions for this section are in \Cref{apx:sparse-hilbert-approximation,apx:rates}.

\subsection{KV compression reduces to sparse balancing in Hilbert space} \label{sec:reduction}

Recall from \Cref{sec:setting} that, for a fixed context measure \(P\) and query distribution \(\nu\), each token affects attention outputs under context reweightings through its response profile \(\Gamma_P(k,v) \in \hilbert_\nu\).
Having formalized the context as a measure, we can form the pushforward
\(
    \mu_{P,\nu} \coloneqq P \circ \Gamma_P^{-1},
\)
which then describes the distribution of responses to the attention numerator and denominator under reweightings of \(P\).
Moreover, every \(K\)-atomic measure \(\hat\mu\) approximating \(\mu_{P,\nu}\) while remaining supported on \(\supp(\mu_{P,\nu})\) can be pulled back to a \(K\)-atomic context measure \(\hat P\).
Consequently, any solution of Problem~\ref{prb:barycenter-sparsification} for the measure \(\mu_{P,\nu}\) in \(\hilbert_\nu\) directly yields a compressed KV cache.
The following proposition shows that error bounds survive this pullback, so that upper bounds for sparse balancing in Hilbert space immediately transfer to bounds on the KV compression error.

\begin{restatable}{proposition}{propreduction} \label{pro:reduction}
    Let \(K \ge 1\), \(P \in \kvmeasures\), and \(\nu \in \quemeasures\). Define \(\mu_{P,\nu} \coloneqq P \circ \Gamma_P^{-1}\).
    Then, \(\mu_{P,\nu}\) is centered and has covariance operator \(\cov_{P,\nu}\).
    Moreover, any fixed \(\hat\mu \in \Prob_K(\supp(\mu_{P,\nu}))\) can be pulled back to a measure \(\hat P \in \kvmeasures[K]\) independently of \(\nu\) such that
    \(
        \err*{P}{\nu}{\hat P}
        \lesssim
        \widetilde\Err_{\mu_{P,\nu}}(\hat\mu)
    \).
    In particular,
    \[
        \inf_{\hat P \in \kvmeasures[K]} \Err_{P,\nu}\bigl(\hat P\bigr)
        \;\;\lesssim\;\;
        \inf_{\hat \mu \in \Prob_K(\supp(\mu_{P,\nu}))} \widetilde\Err_{\mu_{P,\nu}}\bigl(\hat\mu\bigr).
    \]
\end{restatable}

\subsection{Compression bounds from balancing bounds} \label{sec:upper-bounds}
By \Cref{pro:reduction}, it suffices to bound the sparse balancing error in Hilbert space: prove a bound for balancing the barycenter of a probability measure in Hilbert space, and have the same bound apply pointwise to KV compression.
Standard results on sparse approximation in Hilbert spaces already provide basic balancing bounds.
First, kernel herding and conditional-gradient methods give a bound of order \(B^2/K\) when \(\supp(\mu_{P,\nu})\) lies in a radius-\(B\) ball \citep{chen2010super,jaggi2013revisiting}. This can be loose, however, because the radius is sensitive to worst-case response profiles and does not measure data-dependent redundancy.
Second, independent sampling from \(P\) gives a bound of the form \(\trace(\cov_{P,\nu})/K\). This captures the total response variance but does not exploit potential anisotropy of \(\cov_{P,\nu}\), which can be significant when \(\nu\) concentrates on a few important directions.

\subsubsection{Query-aware upper bound}
Since query-aware compressors can explicitly adapt to anisotropic structure in the query distribution, we seek a bound in terms of the spectral decay of \(\cov_{P,\nu}\), which kernel herding and random sampling do not provide.
Below, we provide such a bound, which is new to the best of our knowledge.
The construction works by preserving the barycenter exactly on the leading eigenspace while simultaneously sparsifying the measure on the orthogonal complement.
\Cref{pro:reduction} allows us to immediately transfer this bound to the minimax risk of query-aware KV compression.
\begin{restatable}{theorem}{thmsparsehilbertapproximationfixedk} \label{thm:sparse-hilbert-approximation-fixed-K}
    Fix \(K \ge 7\) and a real Hilbert space \(\hilbert\). Let \(\mu \in \Prob_\fin(\hilbert)\) be centered, with covariance operator \(\Sigma: \hilbert\to\hilbert\).
    Then,
    \(
        \|\int_\hilbert g \,\diff \hat\mu(g)\|_\hilbert^2
        \lesssim
        \frac{\tail_{K/3}(\Sigma)}{K}
    \)
    for some \(\hat\mu \in \Prob_K(\supp(\mu))\).
\end{restatable}
\begin{restatable}{corollary}{corupperaware} \label{cor:upper-bound-aware}
    Let \(K \ge 7\) and \(\mathcal F \subseteq \kvmeasures \times \quemeasures\). Then, \(\minimax{\compspace{K}[aw]}{\mathcal F} \lesssim \sup_{P,\nu \in \mathcal F} \frac{\tail_{K/3}(\cov_{P,\nu})}{K}\).
\end{restatable}

Thus, when \(\nu\) is known, the compression budget can be spent on preserving the leading eigenspace of the response covariance, and the error is controlled only by the response variance outside that space.

\subsubsection{Query-agnostic upper bound}
The above approach relies on knowledge of the geometry of \(\cov_{P,\nu}\).
In contrast, a query-agnostic strategy has no access to this geometry, since it depends on \(\nu\).
A simple balancing strategy that does not rely on knowledge of the Hilbert space geometry is to sample \(K\) atoms independently from \(P\) and take their empirical measure, which corresponds to averaging \(K\) independent centered draws from \(\mu_{P,\nu}\) and yields the following, standard \(\trace(\cov_{P,\nu})/K\) bound for the query-agnostic minimax risk via \Cref{pro:reduction}.

\begin{restatable}{theorem}{thmupperagnostic} \label{thm:upper-bound-agnostic}
    Let \(K \ge 1\) and \(\mathcal F \subseteq \kvmeasures \times \quemeasures\). Then, \(\minimax{\compspace{K}[ag]}{\mathcal F} \lesssim \sup_{P,\nu \in \mathcal F} \trace(\cov_{P,\nu}) / K\).
\end{restatable}

\subsection{Compressors cannot beat the spectral barrier} \label{sec:lower-bounds}

We now show that the upper bounds from \Cref{cor:upper-bound-aware} and \Cref{thm:upper-bound-agnostic} are tight when attention can be sharp.
In this regime, it is known that accurate KV compression, and fast approximate attention more broadly, is impossible \emph{in the worst case} \citep{haris2025compression,alman2023fast,gupta2025subquadratic,haverbeck2026efficiently}.
We generalize these results by showing that sharpness alone does not determine whether a cache is compressible.
Instead, the minimax compression risk is governed by the spectrum of \(\cov_{P,\nu}\) throughout a nontrivial range, from essentially loss-free compression to worst possible error scale under bounded values.

Concretely, we construct admissible contexts, together with query distributions, on which the error for any compression is calibrated to \(\tail_{K/3}(\cov_{P,\nu})\) and \(\trace(\cov_{P,\nu})\), respectively.
Importantly, since an arbitrary finitely supported measure on \(\kvspace\) need not be induced by actual KV caches, our lower-bound instances are chosen from the subclass \(\kvmeasures*[K] \subseteq \kvmeasures[K]\) of context measures induced by KV caches \((k_i,v_i)_{i=1}^K\).
Moreover, the constructed contexts have length \(2K\), so that our lower bounds already apply when the original context is only a constant factor larger than the compression budget and do not rely on compressing extremely long contexts into tiny summaries.

\subsubsection{Query-aware lower bound}
When the compressor knows \(\nu\), we show that, at every scale \(R \in [0,V^2K]\), there exist instances with covariance \(\tail_{K/3}(\cov_{P,\nu}) \asymp R\), on which every \(K\)-atomic compressor incurs error at scale \(R/K\).
Thus, the rate \(\tail_{K/3}(\cov_{P,\nu})/K\) governs the query-aware minimax risk throughout the full natural range, from exact compressibility at zero error to the maximum possible scale \(V^2\).

\begin{restatable}{theorem}{thmlowerboundaware} \label{thm:lower-bound-aware}
    Let \(K \ge 50\) and \(R \in [0, V^2 K]\).
    Suppose \Cref{asp:richness} holds.
    Then, there exists \(\mathcal F \subseteq \kvmeasures*[2K] \times \quemeasures\) with \(\tail_{K/3}(\cov_{P,\nu}) \asymp R\) for all \((P,\nu) \in \mathcal F\) and
    \(
        \minimax{\compspace{K}[aw]}{\mathcal F} \gtrsim \frac{R}{K}
    \).
\end{restatable}

\subsubsection{Query-agnostic lower bound}

When the compressor does not have access to \(\nu\), it must choose a single summary to protect against many query distributions simultaneously.
In this harder setting, we obtain a sharper lower bound.
Specifically, at every scale \(T \in [V^2,V^2K]\), there exist instances with covariance \(\trace(\cov_{P,\nu}) \asymp T\), on which every \(K\)-atomic compressor incurs error at scale \(T/K\) on some query distribution.
At the same time, for each individual query distribution, perfect compression of the same contexts is possible, witnessing a genuine separation of the two regimes.
The price of hiding the query distribution is therefore the loss of access to anisotropic geometry in \(\cov_{P,\nu}\).
Query-aware compressors can preserve the dominant directions and pay only for the tail, whereas query-agnostic compressors must in general pay for the full trace, which can be significantly worse.

\begin{restatable}{theorem}{thmlowerboundagnostic} \label{thm:lower-bound-agnostic}
    Let \(K \ge 50\) and \(T \in [V^2, V^2 K]\).
    Suppose \Cref{asp:richness} holds.
    Then, there exists \(\mathcal F \subseteq \kvmeasures*[2K] \times \quemeasures\) with \(\trace(\cov_{P,\nu}) \asymp T\) for all \((P,\nu) \in \mathcal F\) and
    \(
        \minimax{\compspace{K}[ag]}{\mathcal F} \gtrsim \frac{T}{K},
    \)
    while
    \(
        \minimax{\compspace{K}[aw]}{\mathcal F} = 0
    \).
\end{restatable}
\section{Efficient KV compression algorithms with minimax-optimal risk} \label{sec:algorithms}
We now study how the abstract minimax rates from the previous section can be attained algorithmically during practical Transformer inference, focusing on the causally masked setting of autoregressive sequence models \citep{vaswani2017attention,dai2019transformer,brown2020language}. During \emph{prefill}, the prompt is processed in parallel subject to causal attention.  During \emph{autoregressive decoding}, tokens arrive sequentially and attend to the full past. Our goal is to make efficient and minimax-optimal KV compression compatible with both regimes.

Since the future-query distribution is generally unknown at compression time, we cannot realistically assume full access to it. At the same time, as established in \Cref{sec:minimax-rates}, without any information about future queries nontrivial approximation guarantees are generally impossible to achieve. To overcome this limitation, we allow specifying a finite-dimensional feature map \(\Phi:\kvspace\to\R^d\) to encode aspects of the cached key--value pairs expected to be relevant for future attention, based, for example, on a heuristic or approximate model of the future-query distribution. As our query-agnostic lower bounds imply, no fixed choice of \(\Phi\) can yield improvements for all conceivable query distributions. Nevertheless, we can seek algorithms that attain the query-agnostic minimax risk unconditionally while improving upon it to the extent that \(\Phi\) captures features that are relevant to preserve under the actual query distribution, recovering the query-aware minimax rate for a hypothetically optimal choice of \(\Phi\). Accordingly, we seek to design algorithms that
\begin{enumerate*}[label=\textbf{\textup{(\roman*)}}, itemjoin={{; }}, itemjoin*={{; and }}]
    \item reduce the cost of attention during prefill and autoregressive decoding
    \item reduce memory during autoregressive decoding
    \item respect causal masking
    \item attain the query-agnostic minimax risk
    \item approach the query-aware minimax risk for increasingly good choices of \(\Phi\) for the query distribution.
\end{enumerate*}

\paragraph{Notation.}
We work with a stream \((k_i,v_i)_{i\ge 1}\) of key--value pairs.
The first \(n\) pairs form the prompt; subsequent pairs arrive sequentially during decoding.
Fix a chunk size \(K\), and assume \(n=MK\) with \(M\) a power of two.
For any index set \(I \subseteq \mathbb N\), let \(P_{I}\) denote the context measure on the indexed tokens, and define its unnormalized mass
\(
    w(I) \coloneqq \sum_{i\in I} \exp(\|k_i\|^2/(2\sqrt{d_k}))
\).
If \(P,Q \in \kvmeasures\) represent disjoint index sets \(I\) and \(J\), we write
\[
    \textstyle
    P \oplus Q
    \coloneqq
    \frac{w(I)}{w(I)+w(J)}P
    +
    \frac{w(J)}{w(I)+w(J)}Q
\]
for their union, suppressing the dependence on \(I\) and \(J\) when clear from context.

All proofs and constructions for this section are in \Cref{apx:algorithms}.

\begin{algorithm}[t]
\caption{Parallel prefix compression for causally masked prefill}
\label{alg:prefill-compression}
\begin{algorithmic}[1]
    \State $M \gets n/K$, \quad $L \gets \log_2 M$, \quad $S_L^1 \gets \emptyset$
    \State $Q_0^0, ..., Q_{M-1}^0 \gets C_1, ..., C_M$ \Comment{partition prompt into chunks of size \(K\)}
    \For{$\ell=1,\ldots,L$} \Comment{\textbf{upsweep stage}}
        \ParForAll{$j=1,\ldots,M/2^\ell$}
            \State $Q_j^\ell \gets \reduce(Q_{2j-1}^{\ell-1}\oplus Q_{2j}^{\ell-1})$
            \Comment{summary of tokens $(j-1)2^\ell K+1$ to $j2^\ell K$}
        \EndParForAll
    \EndFor
    \For{$\ell=L,L-1,\ldots,1$} \Comment{\textbf{downsweep stage}}
        \ParForAll{$j=1,\ldots,M/2^\ell$}
            \State $S_{2j-1}^{\ell-1} \gets S_j^\ell$
            \Comment{summary of tokens $1$ to $(2j-2)2^{\ell-1}K$}
            \State $S_{2j}^{\ell-1} \gets \reduce(S_j^\ell \oplus Q_{2j-1}^{\ell-1})$
            \Comment{summary of tokens $1$ to $(2j-1)2^{\ell-1}K$}
        \EndParForAll
    \EndFor
    \State \Return $(S_c \gets S_c^0)_{c=1}^M$ and $Q_1^L$ \Comment{exclusive prefix summaries and full-prompt summary}
\end{algorithmic}
\end{algorithm}
\begin{algorithm}[t]
\caption{Streaming prefix compression for autoregressive decoding}
\label{alg:decoding-compression}
\begin{algorithmic}[1]
    \State $\mathcal B_0,\mathcal B_1,\ldots \gets \emptyset$, \quad $\mathcal B_L \gets Q_1^L$, \quad $\ell \gets 0$
    \Comment{$Q_1^L$ is the summary produced by \Cref{alg:prefill-compression}}
    \For{$t=0,K,2K,\ldots$}
        \State $Q \gets P_{\{n+t+1,...,n+t+K\}}$
        \While{$\mathcal B_\ell\neq\emptyset$} \Comment{update buckets}
            \State $Q \gets \reduce(\mathcal B_\ell\oplus Q)$, \quad $\mathcal B_\ell\gets\emptyset$, \quad $\ell\gets \ell+1$
        \EndWhile
        \State $\ell \gets 0$, \quad $\mathcal B_\ell \gets Q$, \quad $H \gets \emptyset$
        \For{$\ell$ in decreasing order} \Comment{compress buckets to size \(K\)}
            \State $H\gets\reduce(H\oplus\mathcal B_\ell)$
        \EndFor
        \State \textbf{yield} $H$
    \EndFor
\end{algorithmic}
\end{algorithm}

\subsection{Efficient prefill and autoregressive decoding under causal masking} \label{sec:algorithms:computation}
In causally masked prefill, token \(t\) may attend only to the prefix \(P_t \coloneqq P_{\{1,\dots,t\}}\), so compression must provide \emph{prefix summaries} for each position \(t\).
Since storing a separate summary for every prefix would defeat the purpose of compression, we summarize only at chunk boundaries.
Specifically, we fix a chunk size \(K\), partition the stream into chunks \(I_1,I_2,\ldots\) of size \(K\), and maintain a \(K\)-atomic summary \(S_c\) of the prefix for each chunk \(c\).
At position \(t \in I_c\), attention is then computed against
\begin{equation} \label{eq:summary}
    \hat P_t
    \coloneqq
    S_{c-1} \oplus P_{\{\min I_c,\dots,t\}},
\end{equation}
which keeps the local causal prefix exact and compresses only the long-range past.
Hence, every attention computation uses at most \(2K\) atoms while respecting causal masking.

The summaries \(S_c\) are computed by recursively applying a randomized local reducer \(\reduce : \mathcal P_{2K}(\mathcal X) \to \mathcal P_K(\mathcal X)\).
\Cref{alg:prefill-compression} implements this recursion as a parallel Blelloch scan over chunks, producing one
\(K\)-atomic summary for each chunk prefix and one summary of the full prompt.
\Cref{alg:decoding-compression} uses the same merge-reduce primitive online: starting from the prompt summary, it updates the
compressed history during autoregressive decoding and emits a new \(K\)-atomic summary at each chunk boundary.
Thus, both prefill and autoregressive decoding reduce to a small number of calls to \(\reduce\) on inputs of size \(2K\).
Consequently, if \(\reduce\) can be implemented efficiently, the full compression scheme meets our efficiency desiderata, as the following proposition shows.

\begin{restatable}{proposition}{proruntime} \label{pro:runtime}
    \Cref{alg:prefill-compression} makes \(\bigO(n/K)\) calls to \(\reduce\), and uses parallel depth \(\bigO(\log(n/K))\) and \(\bigO(n)\) memory.
    Additionally, at time \(t\), \Cref{alg:decoding-compression} stores \(\bigO(K \, \log(t/K))\) atoms, and makes \(\bigO(\log(t/K)/K)\) calls to \(\reduce\) amortized per token.
\end{restatable}

\subsection{From local reducers to global risk guarantees}
\Cref{alg:prefill-compression,alg:decoding-compression} lift any local compressor \(\reduce:\kvmeasures[2K]\to\kvmeasures[K]\) to a global compression scheme for causally masked prefill and autoregressive decoding.
It remains to identify conditions on \(\reduce\) that yield global risk guarantees.
We give such conditions below.

\begin{restatable}{criteria}{defdesigncriteria} \label{def:admissible-reducer}
    Fix \(K \ge 1\), a Hilbert space \(\hilbert\), and \(\Phi: \kvspace\to\R^d\).
    For \(P \in \kvmeasures\) and \(Q\in \kvmeasures[2K]\), define
    \[
        \textstyle
        \mu_P^\Phi \coloneqq \int_\kvspace \Phi(k,v) \,\diff P(k,v)
        \qquad \text{and} \qquad
        \cov_{P}^\Phi \coloneqq \int_\kvspace (\Phi(k,v) - \mu_P^\Phi) \otimes (\Phi(k,v) - \mu_P^\Phi) \,\diff P(k,v).
    \]
    We call \(\reduce\) \((\Phi,\tau)\)-admissible for \(\tau \ge 0\) if, for all \(Q \in \kvmeasures[2K]\), \(\hat Q \coloneqq \reduce(Q) \in \Prob_K(\supp(Q))\) almost surely and
    \begin{enumerate*}[label=\textbf{\textup{(\roman*)}}, itemjoin={{; }}, itemjoin*={{; and }}]
        \item \(\Ex[\hat Q] = Q\)
        \item \(
            \Ex\bigl[\|\int_{\kvspace}\Psi \diff(\hat Q-Q)\|_{\mathcal G}^2\bigr]
            \lesssim
            \frac{1}{K} \trace(\cov_Q^\Psi)
        \) for all Hilbert spaces \(\mathcal G\) and bounded \(\Psi: \kvspace\to\mathcal G\)
        \item \(
            \Ex[\|\int_{\kvspace}\Phi (\diff Q-Q)\|_2^2] \le \frac{\tau}{K}
        \).
    \end{enumerate*}
\end{restatable}

Reducers satisfying these design principles exist, including for \(\tau=0\), for any fixed-dimensional feature map \(\Phi\), and we provide concrete constructions in \Cref{apx:algorithms}. Together, the conditions make the principles underlying the upper bounds from section \Cref{sec:minimax-rates} compatible with repeated compression. Criterion \textbf{(i)} ensures that repeated compression errors through accumulate gracefully through the merge-reduce tree, criterion \textbf{(ii)} ensures the Monte Carlo basline risk, and criterion \textbf{(iii)} improves upon this by protecting the chosen KV features \(\Phi\) from being perturbed through compression. In partiular, invoking criterion \textbf{(iii)} with the top-eigenspace projection of response profile \(\Gamma_P\), together with the standard Monte Carlo rate in the orthogonal complement, imply the query-aware minimax risk by removing the directions most susceptible to compression.

\begin{restatable}{theorem}{thmalgguarantee} \label{thm:alg-risk}
    Let \(K \ge 1\) and \(\nu\in\quemeasures\), and suppose that \(\reduce\) is \((\Phi,\tau)\)-admissible. Let \(\hat P_t \in \kvmeasures[2K]\) be as in \eqref{eq:summary}. Then,
    \[
        \Ex\err*{P_t}{\nu}{\hat P_t}
        \lesssim
        \frac{\log(2+t/K)}{K}
        \inf_{L\in\mathcal{L}(\R^d,\hilbert_\nu)}
        \bigl[
            \trace(\Cov_{P_t}(\Gamma_{P_t} - L\Phi)) + \|L\|_\mathrm{op}^2\tau
        \bigr].
    \]
\end{restatable}

While our algorithm thus attains the query-aware minimax rate for the \emph{optimal} choice of features, up to a logarithmic degradation with time, it also improves upon Monte Carlo sampling when \(\Phi\) captures only \emph{partial} information about what is relevant to preserve under the future-query distribution. This provides practitioners with a convenient way to incorporate domain knowledge through a suitable feature map, while our algorithm translates this knowledge into end-to-end, causally masked KV compression with corresponding improvements over random sampling.
\section{Experiments} \label{sec:experiments}

We conduct a focused long-context evaluation of the compression scheme from \Cref{sec:algorithms} on the ``long'' subset of LongBench-v2~\citep{bai2025longbench}, a multiple-choice benchmark with contexts ranging from \(128\mathrm{k}\) to \(2\mathrm{M}\) words.
We evaluate Qwen3-32B~\citep{qwen3}.
As a baseline, we use attention over the entire available KV cache, following the official truncation protocol of \citet{bai2025longbench} whenever the full context exceeds the model's context window.
We compare the baseline against our compression scheme from \Cref{sec:algorithms} and against ScissorHands \citep{liu2023scissorhands}, SnapKV \citep{li2024snapkv}, and StreamingLLM \citep{xiao2023efficient} from the literature, each applied independently to each attention head.

Since ScissorHands and SnapKV do not provide efficient parallel inference methods for compressed prefill, we let the literature baselines (including StreamingLLM) use the full KV cache during prefill and only compress during decoding.
In contrast, our compression algorithm is used during both prefill and decoding, so that each head only ever attends to a small subset of the tokens seen by the full-attention baseline.
We evaluate our own compression scheme for two local reducers.
The first is random sampling from the context measure \(Q\), which attains the query-agnostic minimax-optimal rate.
The second protects clusters tokens by their protected features \(\Phi\) and satisfies the design criteria of \Cref{sec:algorithms} when tokens are sufficiently clusterable.
We elaborate on the design of this reducer and the choice of feature map \(\Phi\) in \Cref{apx:algorithms}.
\Cref{apx:experiments} has additional implementation details.

\Cref{tab:longbench-long} reports accuracy with standard error for a compression rate of approximately \(95\%\) relative to the full context window.
Random sampling already provides a strong query-agnostic baseline, consistent with its minimax optimality.
Admissible clustering improves accuracy at \(d=64\), and matches the uncompressed baseline at \(d=128\).
Furthermore, all our local reducers are competitive with the literature baselines \citep{liu2023scissorhands,li2024snapkv,xiao2023efficient} despite using significantly fewer tokens during prefill.
While these results illustrate the viability of the proposed design principles, we emphasize that random sampling already achieves performance comparable to all other methods in our evaluation and standard errors overlap, warranting further empirical validation.

\newcommand{\best}[1]{\mathbf{#1}}
\newcommand{\budget}[1]{\textsf{\footnotesize #1}}

\begin{table}[t]
\centering
\caption{
    Accuracy and standard error on the ``long'' subset of LongBench-v2 \citep{bai2025longbench} using Qwen3-32B \citep{qwen3}.
    ``Full KV'' refers to standard softmax attention using the full, uncompressed KV cache.
    Literature baselines \citep{liu2023scissorhands,li2024snapkv,xiao2023efficient} use the full, uncompressed KV cache during prefill, then compress it once after prefill, and use the compressed cache during decoding.
    Our methods compress during both prefill and decoding, using the scheme from \Cref{sec:algorithms} for different choices of the $\reduce$ subroutine, namely random sampling and balanced clustering with different feature dimensions \(d=64\) and \(d=128\).
}
\label{tab:longbench-long}
\begin{tabular*}{\textwidth}{@{\extracolsep{\fill}}llccc@{}}
    \toprule
    &
    \textbf{Method}
    & \multicolumn{2}{c}{\textbf{KV budget}}
    & \textbf{Accuracy (\%)} \\
    \cmidrule(lr){3-4}
    &
    & {\footnotesize\textbf{Prefill}}
    & {\footnotesize\textbf{Decode}}
    & \\
    \midrule

    \multirow{4}{*}{%
        \rotatebox[origin=c]{90}{\textbf{Baselines }}%
    }
    & Full KV
    & \budget{130{,}944}
    & \budget{130{,}944}
    & $37.04\;(\pm 4.65)$
    \\
    \cmidrule(l){2-5}

    & ScissorHands \citep{liu2023scissorhands}
    & \budget{130{,}944}
    & \budget{6152}
    & $35.19\;(\pm 4.60)$
    \\

    & SnapKV \citep{li2024snapkv}
    & \budget{130{,}944}
    & \budget{6152}
    & $34.26\;(\pm 4.57)$
    \\

    & StreamingLLM \citep{xiao2023efficient}
    & \budget{130{,}944}
    & \budget{6152}
    & $32.41\;(\pm 4.50)$
    \\
    \midrule

    \multirow{3}{*}{%
        \rotatebox[origin=c]{90}{\textbf{Ours}}%
    }
    & Random
    & \budget{6152}
    & \budget{6152}
    & $34.26\;(\pm 4.57)$
    \\

    & Balanced clustering ($d=64$)
    & \budget{6152}
    & \budget{6152}
    & $35.19\;(\pm 4.60)$
    \\

    & Balanced clustering ($d=128$)
    & \budget{6152}
    & \budget{6152}
    & $37.04\;(\pm 4.65)$
    \\

    \bottomrule
\end{tabular*}
\end{table}
\section{Conclusion}

We posit KV compression as sparse balancing of a finite probability measure in a Hilbert space, and characterize its minimax risk in terms of how a cache interacts with future queries through softmax attention.
This yields tight rates under natural assumptions and identifies what compressed summaries must preserve.
The same characterization reveals design principles for causal compression during prefill and autoregressive decoding, separating universal guarantees from structure-exploiting refinements.
Together, our results provide a principled path toward practical KV cache compression with theoretical guarantees.

\newpage
\begin{ack}
    Funded by the European Union.
    This work has received funding from the  European High Performance Computing Joint Undertaking (JU) and from the  German Federal Ministry of Research, Technology and Space (BMFTR), the  Ministry of Culture and Science of North Rhine-Westphalia (MKW NRW) and  the Hessian Ministry of Science and Research, Arts and Culture (HMWK)  under grant agreement No \texttt{101250682}.
    LH was supported by the RWTH Research Ambassador Scholarship during his research stay at Stanford University.
    CAA is supported by a Schmidt Science Fellowship.
    Computations were performed with computing resources granted by RWTH Aachen University under project \texttt{thes2228}.
\end{ack}

\bibliographystyle{plainnat}
\bibliography{references}


\newpage
\appendix
\crefalias{section}{appendix}

\startcontents[appendix]

\section{Sparse barycenter balancing in Hilbert space} \label{apx:sparse-hilbert-approximation}
This appendix records an upper bound for the abstract balancing problem in \Cref{prb:barycenter-sparsification}.

\begin{theorem}\label{thm:sparse-hilbert-approximation}
    Let \(\hilbert\) be a real Hilbert space, and let
    \[
        \textstyle
        \mu = \sum_{i=1}^n p_i \,\delta_{x_i} \in \Prob_n(\hilbert)
    \]
    have mean zero and covariance operator
    \[
        \textstyle
        \cov \coloneqq \sum_{i=1}^n p_i\, x_i \otimes x_i.
    \]
    Let \(U \subseteq \hilbert\) be a finite-dimensional subspace. Let \(d \coloneqq \dim U\), and let \(\Pi_U: \hilbert \to \hilbert\) be the orthogonal projector onto \(U\).
    Then, for every \(m \in \N\), there exists \(\hat\mu \in \Prob_{m + 2d + 2}(\supp(\mu))\) with
    \[
        \int_\hilbert x\, d\hat\mu \in U^\perp
        \qquad\text{and}\qquad
        \left\|\int_\hilbert x\, d\hat\mu\right\|_\hilbert^2
        \le
        \frac{20}{9m}\trace\bigl((I - \Pi_U) \cov\bigr).
    \]
\end{theorem}
\begin{proof}
    Deferred to \Cref{apx:sparse-hilbert-approximation:proof}.
\end{proof}

While \Cref{thm:sparse-hilbert-approximation} only claims existence, we give a constructive proof with concrete algorithms.
We later invoke this construction when we design KV compression algorithms in \Cref{apx:algorithms}.

For our construction, we reduce sparsifying a finitely supported measure \(\mu = \sum_{i=1}^n p_i \,\delta_{x_i} \in \Prob_n(\hilbert)\) to a \emph{constrained rounding problem}: Given a vector \((p_1, ..., p_n) \in [0,1]^n\), find a (random) rounded vector \((\hat p_1, ..., \hat p_n) \in [0,1]^n\) with at most \(m+2d+2\) fractional entries preserving the mass of \(p\) while satisfying linear constraints.
As we show, mass preservation is sufficient to construct a sparse measure supported on at most \(m+2d+2\) points from such a constrained rounding procedure, and linear constraints are used to encode exact preservation of the barycenter in the subspace \(U\).

In \Cref{apx:constrained-partial-rounding}, we first describe a general algorithm for constrained rounding with mass preservation.
\Cref{apx:protected-hilbert-rounding} then specializes this algorithm to sparsification of measures with a protected subspace, and \Cref{apx:sparse-hilbert-approximation:proof} proves \Cref{thm:sparse-hilbert-approximation} from this.

\subsection{Finite-dimensional constrained rounding} \label{apx:constrained-partial-rounding}

We first isolate a finite-dimensional constrained rounding procedure.
Starting from a fractional weight vector, the goal is to round most coordinates while preserving mass and satisfying linear constraints.
Such a procedure can be obtained as a straightforward application of the sub-isotropic rounding walk of \citet{bansal2024generalization}.
Given a current point \(x\in[0,1]^n\), an alive set \(A\subseteq[n]\), and linear constraints, we write
\[
    x^+
    =
    \textsc{SubIsoWalk}\!\left(
        x,A
        \;\middle|\;
        \sum_{s=1}^n z_s=m,\;
        \sum_{s=1}^n z_sc_s=0
    \right)
\]
for the stopped phase of Bansal's walk, restricted to the coordinates in \(A\), that returns a random point \(x^+\in[0,1]^n\) satisfying the constraints
\[
    \sum_{s=1}^n x^+_s=m,
    \qquad
    \sum_{s=1}^n x^+_sc_s=0,
\]
and \(x^+_s=x_s\) for all \(s \not\in A\), such that at least one coordinate in \(A\) reaches \(0\) or \(1\).
\Cref{alg:constrained-partial-round} simply applies this procedure iteratively until only \(2d+2\) fractional coordinates remain, which yields the following guarantees on the output.

\begin{algorithm}[t]
\caption{\(\textsc{ConstrainedPartialRound}\)}
\label{alg:constrained-partial-round}
\begin{algorithmic}[1]
    \Require \(z\in[0,1]^N\), vectors \(c_1,\ldots,c_N\in\mathbb R^d\)
    \Require \(\sum_{s=1}^N z_s=m\) and \(\sum_{s=1}^N z_s \,   c_s=0\)
    \State \(t\gets 0, \quad z^{(0)} \gets z\)
    \State \(A\gets \{s\in[N] \;\mid\; 0<z_s^{(0)}<1\}\) \Comment{``alive set''}
    \While{\(|A|>2d+2\)}
        \State \(t\gets t+1\)
        \vspace{0.5em}
        \State \(
            z^{(t)}
            \gets
            \textsc{SubIsoWalk}\!\left(
                z^{(t-1)},A
                \,\mid\,
                \sum_{s=1}^N z_s=m,
                \sum_{s=1}^N z_s c_s=0
            \right)
        \)
        \vspace{0.5em}
        \Comment{provided by \citep{bansal2024generalization}}
        \State \(A\gets \{s\in[N] \;\mid\; 0<z_s^{(t)}<1\}\)
    \EndWhile
    \State \Return \(z^{(t)}\)
\end{algorithmic}
\end{algorithm}

\FloatBarrier

\begin{lemma} \label{lem:partial-constrained-round}
    Fix \(n,d,m \in \N\).
    Let \(z \in [0,1]^n\) and \(c_1, ..., c_n \in \R^d\).
    Assume that \(\sum_{s=1}^n z_s \, c_s = 0\) and \(m \coloneqq \sum_{s=1}^n z_s \in \N\).
    Let \(Z \coloneqq \textsc{ConstrainedPartialRound}(z, \, c_1, ..., c_n)\) be the random output of \Cref{alg:constrained-partial-round}.
    Then, the following hold:
    \begin{enumerate}[label=(\roman*)]
        \item \(\sum_{s=1}^n Z_s = m\) almost surely
        \item \(\sum_{s=1}^n Z_s \, c_s = 0\) almost surely
        \item \(|\{s \in [n] \mid 0 < Z_s < 1\}| \le 2d+2\) almost surely
        \item \(\Ex[Z] = z\)
        \item for every \(a \in \R^n\),
        \[
            \Ex\!\left[\Bigl(\sum_{s=1}^n (Z_s - z_s) \, a \Bigr)^2\right]
            \le
            \frac{20}{9}\sum_{s=1}^n a_s^2 z_s (1-z_s).
        \]
    \end{enumerate}
\end{lemma}
\begin{proof}
    For \(x\in[0,1]^n\), let
    \[
        A(x)\coloneqq \{s\in[n] : x_s\in(0,1)\}
    \]
    be the alive set.
    The algorithm runs the sub-isotropic rounding walk of \citet[Section~3.1]{bansal2024generalization} with the following deterministic schedule of active constraints.
    As long as \(|A(x)|>2d+2\), we keep the constraints
    \[
        \sum_{s=1}^n x_s=m
        \qquad \text{and} \qquad
        \sum_{s=1}^n x_sc_s=0
    \]
    active on the alive coordinates.
    Once \(|A(x)|\le 2d+2\), the algorithm stops.

    We first check that the active constraints satisfy the rank condition required by \citet{bansal2024generalization}.
    Restricted to the alive coordinates, the active system has at most \(d+1\) scalar constraints.
    Whenever the loop is entered, \(|A(x)|>2d+2\), and therefore
    \[
        d+1 \le \frac{|A(x)|}{2}.
    \]
    Thus the slack hypothesis in \citep[Theorem~1.2]{bansal2024generalization} holds with parameter \(\delta=1/2\).

    Let \((z^{(k)})_{k\ge 0}\) denote the elementary iterates of the resulting walk, and let \((\mathcal F_k)_{k\ge 0}\) be its natural filtration.
    Let \(T\) be the terminal time of the full walk and write
    \[
        X \coloneqq z^{(T)} \in \{0,1\}^n
    \]
    for the final integral vector.
    By the martingale property of Bansal's walk, \(\Ex[X]=z\).
    Let
    \[
        \tau
        \coloneqq
        \inf\left\{
            k\ge 0 :
            \bigl|A(z^{(k)})\bigr|\le 2d+2
        \right\}.
    \]
    Then, \(Z\coloneqq z^{(\tau)}\).

    Before time \(\tau\), every increment lies in the nullspace of the active constraints.
    Hence, the identities
    \[
        \sum_{s=1}^n z_s^{(k)}=m
        \qquad\text{and}\qquad
        \sum_{s=1}^n z_s^{(k)}c_s=0
    \]
    are preserved pathwise for all \(k\le\tau\).
    Since they hold at initialization, they hold for \(Z\), which proves
    \[
        \sum_{s=1}^n Z_s=m,
        \qquad
        \sum_{s=1}^n Z_sc_s=0 .
    \]
    The definition of \(\tau\) gives
    \[
        \bigl|\{s\in[n] \mid 0<Z_s<1\}\bigr|\le 2d+2 .
    \]

    It remains to prove unbiasedness and the scalar second-moment bound.
    For \(k\ge 1\), set
    \[
        \Delta^{(k)}\coloneqq z^{(k)}-z^{(k-1)},
    \]
    and put \(\Delta^{(k)}\coloneqq 0\) for \(k>T\).
    In the notation of \citet[Section~3.1]{bansal2024generalization}, conditional on \(\mathcal F_{k-1}\), the increment on the alive coordinates is of the form \(\gamma_k U_k^{1/2}r_k\), where \(r_k\) is a random sign vector and \(U_k\) is positive semidefinite with range contained in the nullspace of the active constraints.
    Embedding the conditional covariance into \(\mathbb R^n\) and padding zeros outside the alive coordinates gives a matrix \(\widetilde U_k\) such that
    \[
        \Ex[\Delta^{(k)}\mid\mathcal F_{k-1}]=0,
        \qquad
        \Ex\!\left[
            \Delta^{(k)}(\Delta^{(k)})^\top
            \mid\mathcal F_{k-1}
        \right]
        =
        \widetilde U_k .
    \]
    Moreover, the sub-isotropic covariance estimate in the same construction gives, for \(\delta=1/2\),
    \[
        \widetilde U_k
        \preceq
        \frac{20}{9}\operatorname{diag}(\widetilde U_k).
    \]

    Fix \(a\in\mathbb R^n\).
    Since \((\langle a,z^{(k)}\rangle)_{k\ge 0}\) is a martingale, its increments are orthogonal in \(L_2\).
    Hence,
    \[
        \Ex\!\left[
            \bigl\langle a,X-z\bigr\rangle^2
        \right]
        =
        \sum_{k\ge 1}
        \Ex\!\left[
            \bigl\langle a,\Delta^{(k)}\bigr\rangle^2
        \right].
    \]
    Using the conditional covariance estimate,
    \begin{align*}
        \Ex\!\left[
            \bigl\langle a,\Delta^{(k)}\bigr\rangle^2
            \mid \mathcal F_{k-1}
        \right]
        &=
        a^\top \widetilde U_k a  \\
        &\le
        \frac{20}{9}
        \sum_{s=1}^n a_s^2(\widetilde U_k)_{ss} \\
        &=
        \frac{20}{9}
        \sum_{s=1}^n a_s^2
        \Ex\!\left[
            (\Delta_s^{(k)})^2
            \mid \mathcal F_{k-1}
        \right].
    \end{align*}
    Summing over \(k\) and taking expectations yields
    \[
        \Ex\!\left[
            \bigl\langle a,X-z\bigr\rangle^2
        \right]
        \le
        \frac{20}{9}
        \sum_{s=1}^n a_s^2
        \sum_{k\ge 1}\Ex\!\left[(\Delta_s^{(k)})^2\right].
    \]
    For each coordinate \(s\), the process \((z_s^{(k)})_{k\ge0}\) is a bounded martingale with terminal value \(X_s\in\{0,1\}\).
    Its increments are orthogonal in \(L_2\), and therefore
    \[
        \sum_{k\ge 1}\Ex\!\left[(\Delta_s^{(k)})^2\right]
        =
        \Ex\!\left[(X_s-z_s)^2\right]
        =
        z_s(1-z_s),
    \]
    where we used \(X_s^2=X_s\) and \(\Ex[X_s]=z_s\).
    Consequently,
    \[
        \Ex\!\left[
            \bigl\langle a,X-z\bigr\rangle^2
        \right]
        \le
        \frac{20}{9}
        \sum_{s=1}^n a_s^2 z_s(1-z_s).
    \]

    Finally, the stopped process \((z^{(k\wedge T)})_{k\ge0}\) is a bounded martingale and \(\tau\le T\) almost surely.
    Hence, optional stopping gives
    \[
        Z_s
        =
        z_s^{(\tau)}
        =
        \Ex[X_s\mid\mathcal F_\tau]
    \]
    for every \(s\in[n]\).
    Taking expectations gives \(\Ex[Z]=z\).
    Moreover,
    \[
        \bigl\langle a,Z-z\bigr\rangle
        =
        \Ex\!\left[
            \bigl\langle a,X-z\bigr\rangle
            \mid\mathcal F_\tau
        \right],
    \]
    and conditional Jensen yields
    \[
        \Ex\!\left[
            \bigl\langle a,Z-z\bigr\rangle^2
        \right]
        \le
        \Ex\!\left[
            \bigl\langle a,X-z\bigr\rangle^2
        \right]
        \le
        \frac{20}{9}
        \sum_{s=1}^n a_s^2 z_s(1-z_s).
    \]
    That proves the claim.
\end{proof}

\subsection{Protected Hilbert rounding} \label{apx:protected-hilbert-rounding}

We now apply the constrained rounding procedure to finitely supported measures on a Hilbert space.
The linear constraints are chosen to preserve the barycenter in a protected subspace, and mass preservation forces small support.

\begin{algorithm}[t]
\caption{\(\textsc{ProtectedSparsify}\)}
\label{alg:protected-hilbert-round}
\begin{algorithmic}[1]
    \Require \(\mu=\sum_{j=1}^S p_j \, \delta_{x_j}\in\Prob_{\fin}(\hilbert)\)
    \Require finite-dimensional subspace \(U\subseteq\hilbert\) with \(d=\dim(U)\)
    \Require integer \(m\ge 1\)
    \State \(\bar x\gets \int x \diff\mu\)
    \State \(u_1,...,u_d \gets \text{orthonormal basis of \ \(U\)}\)

    \vspace{0.5em}
    \State \(s\gets 0\) \Comment{split atoms to have small mass}
    \For{\(j=1,\ldots,S\)}
        \State \(n_j\gets \lceil m p_j\rceil\)
        \For{\(\ell=1,\ldots,n_j\)}
            \State \(s\gets s+1\)
            \State \(y_s\gets x_j\)
            \State \(\iota_s\gets j\)
            \State \(\alpha_s\gets p_j/n_j\)
        \EndFor
    \EndFor
    \State \(N\gets s\)

    \vspace{0.5em}
    \For{\(s=1,\ldots,N\)} \Comment{compute protected coordinates}
        \State \(\xi_s\gets y_s-\bar x\)
        \State \(c_s\gets \bigl(\langle \xi_s,u_1\rangle,\ldots,\langle \xi_s,u_d\rangle\bigr)\in\mathbb R^d\)
        \State \(z_s\gets m\alpha_s\)
    \EndFor
    
    \vspace{0.5em}
    \State \(Z\gets \textsc{ConstrainedPartialRound}\!\left(z, \, c_1,\ldots,c_N\right)\) \Comment{sparsify}
    \vspace{0.5em}
    
    \For{\(j=1,\ldots,S\)} \Comment{reassemble atoms}
        \State \(\hat p_j\gets \frac1m\sum_{s:\,\iota_s=j} Z_s\)
    \EndFor
    
    \vspace{0.5em}
    \State \(\hat\mu\gets \sum_{j=1}^S \hat p_j \, \delta_{x_j}\)
    \State \Return \(\hat\mu\)
\end{algorithmic}
\end{algorithm}

\FloatBarrier

\begin{proposition} \label{pro:protected-sparsification}
    Fix a Hilbert space \(\hilbert\), a finite-dimensional subspace \(U\subseteq\hilbert\) with \(d\coloneqq\dim(U)\), and some \(m\ge1\).
    Let \(\mu = \sum_{j=1}^S p_j \, \delta_{x_j} \in \Prob_{\fin}(\hilbert)\) and let 
    \[
        \textstyle
        \hat\mu \coloneqq \textsc{ProtectedSparsify}(\mu,U,m) = \sum_{j=1}^S \hat p_j \, \delta_{x_j}
    \]
    be the random output of \Cref{alg:protected-hilbert-round}. Write
    \[
        \textstyle
        \bar x\coloneqq \int_\hilbert x \diff\mu,
        \qquad
        \cov
        \coloneqq
        \int_\hilbert (x-\bar x)\otimes(x-\bar x) \diff\mu,
        \qquad
        \eta
        \coloneqq
        \int_\hilbert x \diff(\hat\mu - \mu).
    \]
    Then, the following hold:
    \begin{enumerate}[label=(\roman*)]
        \item \(\hat\mu\in\Prob_{m+2d+2}(\supp(\mu))\) almost surely
        \item \(\int_\hilbert x \diff(\hat \mu - \mu) \in U^\perp\) almost surely
        \item \(\Ex[\hat\mu] = \mu\)
        \item \(\Ex[\eta \otimes \eta] \preceq \frac{20}{9m}(I-\Pi_U)\cov(I-\Pi_U)\)
    \end{enumerate}
    Additionally, for every Hilbert space \(\mathcal G\) and every \(y_1,\ldots,y_S\in\mathcal G\), writing
    \[
        \textstyle
        \bar y\coloneqq \sum_{j=1}^S p_jy_j,
        \qquad
        \cov_y\coloneqq
        \sum_{j=1}^S p_j(y_j-\bar y)\otimes(y_j-\bar y),
    \]
    one has
    \[
        \Ex \Bigl\| \sum_{j=1}^S(\hat p_j-p_j)y_j \Bigr\|_{\mathcal G}^2
        \le
        \frac{20}{9m}\trace(\cov_y).
    \]
\end{proposition}
\begin{proof}
    Use the notation of \Cref{alg:protected-hilbert-round}.
    We verify the four guarantees in the order suggested by the construction: first that the finite-dimensional rounding lemma applies, then support preservation and unbiasedness, then exact cancellation on \(U\), and finally the operator second-moment bound.
    
    The splitting step gives
    \[
        \sum_{s=1}^N\alpha_s\delta_{y_s}=\mu,
        \qquad
        0<\alpha_s\le \frac1m .
    \]
    Hence \(z_s=m\alpha_s\in[0,1]\) and \(\sum_s z_s=m\).
    Since
    \[
        \sum_{s=1}^N\alpha_s\xi_s
        =
        \int_\hilbert (x-\bar x) \diff\mu
        =
        0,
    \]
    the protected coordinates satisfy
    \[
        \sum_{s=1}^N z_sc_s=0.
    \]
    Thus, \(\textsc{ConstrainedPartialRound}\) applies.
    By \Cref{lem:partial-constrained-round}, it returns \(Z\in[0,1]^N\) with
    \[
        \sum_{s=1}^N Z_s=m,
        \qquad
        \sum_{s=1}^N Z_sc_s=0,
        \qquad
        \Ex[Z]=z,
    \]
    at most \(2d+2\) fractional coordinates, and, for every \(a\in\R^N\),
    \[
        \Ex\!\left[
            \left(\sum_{s=1}^N (Z_s-z_s)a_s\right)^2
        \right]
        \le
        \frac{20}{9}
        \sum_{s=1}^N z_s(1-z_s)a_s^2 .
    \]

    The non-fractional positive coordinates of \(Z\) are equal to \(1\), and their count is at most \(\sum_s Z_s=m\).
    Together with the at most \(2d+2\) fractional coordinates, this shows that at most \(m+2d+2\) split copies receive positive mass.
    After reassembling split copies, the same bound holds for the number of original atoms with \(\hat p_j>0\).
    Hence,
    \[
        \hat\mu\in\Prob_{m+2d+2}(\supp(\mu))
    \]
    almost surely.
    
    Moreover, for each \(j\),
    \[
        \Ex[\hat p_j]
        =
        \frac1m\sum_{s:\,\iota_s=j}\Ex[Z_s]
        =
        \frac1m\sum_{s:\,\iota_s=j}z_s
        =
        \sum_{s:\,\iota_s=j}\alpha_s
        =
        p_j.
    \]
    Therefore,
    \[
        \Ex[\hat\mu]
        =
        \sum_{j=1}^S \Ex[\hat p_j]\delta_{x_j}
        =
        \sum_{j=1}^S p_j\delta_{x_j}
        =
        \mu.
    \]

    By the reassembly definition of \(\hat p_j\),
    \[
        \eta
        =
        \sum_{j=1}^S(\hat p_j-p_j)x_j
        =
        \frac1m\sum_{s=1}^N (Z_s-z_s)y_s .
    \]
    Since \(\sum_s Z_s=\sum_s z_s=m\), the constant component \(\bar x\) cancels, and hence
    \[
        \eta
        =
        \frac1m\sum_{s=1}^N (Z_s-z_s)\xi_s.
    \]
    Moreover, \(\sum_s(Z_s-z_s)c_s=0\), and therefore
    \[
        \langle \eta,u_\ell\rangle
        =
        \frac1m\sum_{s=1}^N
        (Z_s-z_s)\langle \xi_s,u_\ell\rangle
        =
        0
        \qquad
        (\ell=1,\ldots,d).
    \]
    Thus \(\eta\in U^\perp\), and consequently
    \[
        \eta
        =
        \frac1m\sum_{s=1}^N
        (Z_s-z_s)(I-\Pi_U)\xi_s.
    \]

    Fix \(h\in\hilbert\), and apply the scalar second-moment bound with \(a_s=\langle h,(I-\Pi_U)\xi_s\rangle\).
    Since \(z_s=m\alpha_s\),
    \begin{align*}
        \Ex\langle h,\eta\rangle^2
        &\le
        \frac{20}{9m^2}
        \sum_{s=1}^N
        z_s(1-z_s)
        \langle h,(I-\Pi_U)\xi_s\rangle^2 \\
        &\le
        \frac{20}{9m}
        \sum_{s=1}^N
        \alpha_s
        \langle h,(I-\Pi_U)\xi_s\rangle^2 \\
        &=
        \frac{20}{9m}
        \left\langle
            h,(I-\Pi_U)\Sigma_\mu(I-\Pi_U)h
        \right\rangle .
    \end{align*}
    Since \(h\in\hilbert\) was arbitrary, the operator inequality follows.

    It remains to prove the auxiliary variance bound. Since \(\sum_j(\hat p_j-p_j)=0\),
    \[
        \sum_{j=1}^S(\hat p_j-p_j)y_j
        =
        \sum_{j=1}^S(\hat p_j-p_j)(y_j-\bar y).
    \]
    Using the split copies,
    \[
        \sum_{j=1}^S(\hat p_j-p_j)(y_j-\bar y)
        =
        \frac1m
        \sum_{s=1}^N
        (Z_s-z_s)(y_{\iota_s}-\bar y).
    \]
    Fix \(g\in\mathcal G\) and apply the scalar second-moment bound with
    \[
        a_s
        =
        \langle g,y_{\iota_s}-\bar y\rangle_{\mathcal G}.
    \]
    Then
    \[
        \Ex\left[
            \left\langle
                g,
                \sum_{j=1}^S(\hat p_j-p_j)y_j
            \right\rangle^2
        \right]
        \le
        \frac{20}{9m^2}
        \sum_{s=1}^N
        z_s(1-z_s)
        \langle g,y_{\iota_s}-\bar y\rangle^2.
    \]
    Since \(z_s=m\alpha_s\le1\),
    \[
        z_s(1-z_s)\le m\alpha_s.
    \]
    Therefore,
    \[
        \Ex \left\langle g, \sum_{j=1}^S(\hat p_j-p_j)y_j \right\rangle^2
        \le
        \frac{20}{9m}
        \sum_{s=1}^N
        \alpha_s
        \langle g,y_{\iota_s}-\bar y\rangle^2.
    \]
    Reassembling the split copies gives
    \[
        \sum_{s=1}^N
        \alpha_s
        (y_{\iota_s}-\bar y)\otimes(y_{\iota_s}-\bar y)
        =
        \cov_y.
    \]
    Thus
    \[
        \Ex \left\langle g, \sum_{j=1}^S(\hat p_j-p_j)y_j \right\rangle^2
        \le
        \frac{20}{9m}
        \langle g,\cov_y g\rangle.
    \]
    Taking traces yields the claim.
\end{proof}

\subsection[Proof of Theorem~\ref{thm:sparse-hilbert-approximation}]{Proof of \Cref{thm:sparse-hilbert-approximation}} \label{apx:sparse-hilbert-approximation:proof}
\Cref{thm:sparse-hilbert-approximation} now follows by running \textsc{ProtectedSparsify} on the measure \(\mu\) and subspace \(U\), and taking one random realization on which the desired bound holds.

\begin{proof}
    Apply \Cref{pro:protected-sparsification} to \((\mu,U,m)\), and let \(\hat\mu\) denote the random output.
    Since \(\mu\) is centered,
    \[
        \eta
        \coloneqq
        \int_\hilbert x \diff\hat\mu
        =
        \int_\hilbert x \diff(\hat\mu - \mu).
    \]
    Hence, \Cref{pro:protected-sparsification} gives a random \(\hat\mu\) such that
    \begin{enumerate}[label=(\roman*)]
        \item \label{itm:sparse-hilbert-approximation:1} \(\hat\mu\in\Prob_{m+2d+2}(\supp(\mu))\) almost surely;
        \item \label{itm:sparse-hilbert-approximation:2} \(\int_\hilbert x \diff\hat \mu \in U^\perp\) almost surely; and
        \item \label{itm:sparse-hilbert-approximation:3} \(\Ex[\eta \otimes \eta] \preceq \frac{20}{9m}(I-\Pi_U)\cov(I-\Pi_U)\).
    \end{enumerate}

    Moreover, taking traces in \ref{itm:sparse-hilbert-approximation:3} gives
    \[
        \Ex\left\|\int_\hilbert x \diff\hat\mu\right\|_\hilbert^2
        =
        \Ex\|\eta\|_\hilbert^2
        \le
        \frac{20}{9m}
        \trace\bigl((I-\Pi_U)\cov(I-\Pi_U)\bigr)
        =
        \frac{20}{9m}
        \trace\bigl((I-\Pi_U)\cov\bigr).
    \]
    Therefore some realization of \(\hat\mu\) satisfies the same bound.
\end{proof}
\newpage
\section{Minimax rates} \label{apx:rates}
This appendix proves our minimax rates for KV compression in the query-aware and query-agnostic settings. \Cref{apx:upper-bounds} proves upper bounds and \Cref{apx:lower-bounds} proves the matching lower bounds.

\paragraph{Notation.}
We briefly recall the relevant notation from the main text.
Given a context measure \(P \in \kvmeasures\), we write softmax attention as
\[
    \att*{q}{P}
    =
    \int_{\kvspace} a_P(q,k)\, v \, \diff P(k,v),
    \qquad
    a_k(q \mid P)
    =
    \frac{\kappa(q,k)}{\int_{\kvspace} \kappa(q,k') \, \diff P(k',v')}
\]
and study the mean squared error of compressed context measures \(\hat P \in \kvmeasures[K]\)
\[
    \err*{P}{\nu}{\hat P}
    =
    \Ex_{q \sim \nu}\!\left\|\att*{q}{P} - \att*{q}{\hat P}\right\|_2^2
\]
under a fixed query distribution \(\nu \in \quemeasures\).
The response profile of a key--value pair \((k,v) \in \kvspace\) under \(P\) is
\[
    \Gamma_P(k,v) : \quespace \to \R^{d_v}\oplus\R,
    \qquad
    q \mapsto
    \Bigl(
        a_k(q \mid P)\bigl(v-\att{q}{P}\bigr),
        \;
        V\bigl(a_k(q \mid P)-1\bigr)
    \Bigr),
\]
with the corresponding covariance operator
\[
    \cov_{P,\nu} = \int_\kvspace \Gamma_P(k,v) \otimes \Gamma_P(k,v) \,\diff P(k,v),
\]
where \(\Gamma_P(k,v) \otimes \Gamma_P(k,v)\) is understood as the rank-one operator
\[
    h \mapsto \dotp{\Gamma_P(k,v)}{h}_{\hilbert_\nu} \cdot \Gamma_P(k,v)
\]
and \(\hilbert_\nu = L_2(\nu, \R^{d_v} \oplus \R)\).

\subsection{Upper bounds} \label{apx:upper-bounds}
We derive upper bounds on the minimax compression risk in the query-aware and query-agnostic settings from a shared Hilbert space representation.
For fixed \((P,\nu)\), the response profile map \(\Gamma_P\) sends each token \((k,v)\) to an element of the Hilbert space \(\hilbert_\nu\), and thereby pushes the context measure \(P\) forward to a finitely supported probability measure on \(\hilbert_\nu\).
This measure is centered and has covariance \(\cov_{P,\nu}\), and the attention error of a compressed context is controlled by the squared norm of the corresponding barycenter in \(\hilbert_\nu\).
The query-aware and query-agnostic regimes differ only in how this barycenter is approximated: in the aware regime, we may choose a sparse approximation after seeing \(\nu\), whereas in the agnostic regime we must rely on a query-independent approximation.
Both regimes permit upper bounds on the sparse barycenter approximation problem, which we then translate back to bounds on the KV compression problem.

\subsubsection{Reduction to barycenter sparsification}
Recall the surrogate error definition
\[
    \widetilde\Err_{\mu}(\hat \mu)
    =
    \left\|\int_{\hilbert} g \,\diff\hat\mu(g)\right\|_{\hilbert}^2
\]
between measures \(\mu\) and \(\hat \mu\) on a Hilbert space \(\hilbert_\nu\).
Our goal is to derive upper bounds on the compression error \(\Err_{P,\nu}(\hat P)\) from upper bounds on \(\widetilde\Err(\hat\mu)\) for the pushforward \(\mu_{P,\nu} = P \circ \Gamma_P^{-1}\) by pulling back a sparse approximation \(\hat\mu \in \Prob_K(\supp(\mu_{P,\nu}))\) of \(\mu_{P,\nu}\) to a sparse approximation \(\hat P \in \kvmeasures[K]\) of \(P\).
\Cref{lem:error-hilbert-transfer} first shows that upper bounds on \(\widetilde\Err\) indeed survive this pullback, and \Cref{pro:reduction} packages the full reduction principle.

\begin{lemma} \label{lem:error-hilbert-transfer}
    Let \(P \in \kvmeasures\), \(\nu \in \quemeasures\), and let \(\hat P \in \Prob(\supp(P))\).
    Then,
    \[
        \err*{P}{\nu}{\hat P}
        \le
        16 \left\|
            \int_{\kvspace} \Gamma_P(k,v)\,\diff \hat P(k,v)
        \right\|_{\hilbert_\nu}^2.
    \]
\end{lemma}

\propreduction*

\begin{proof}[Proof of \Cref{lem:error-hilbert-transfer}]
    Fix \(q \in \supp(\nu)\), and abbreviate
    \[
        m_P(q) \coloneqq \int_{\kvspace} \kappa(q,k)\,\diff P(k,v),
        \qquad
        m_{\hat P}(q) \coloneqq \int_{\kvspace} \kappa(q,k)\,\diff \hat P(k,v).
    \]
    Define
    \[
        U(q)
        \coloneqq
        \int_{\kvspace} a_k(q \mid P)\bigl(v-\att*{q}{P}\bigr)\,\diff \hat P(k,v),
    \]
    and
    \[
        t(q)
        \coloneqq
        \int_{\kvspace} \bigl(a_k(q \mid P)-1\bigr)\,\diff \hat P(k,v).
    \]
    Since \(\hat P\) is a probability measure,
    \[
        1+t(q)
        =
        \int_{\kvspace} a_k(q \mid P)\,\diff \hat P(k,v)
        =
        \frac{m_{\hat P}(q)}{m_P(q)}
        \ge 0.
    \]
    Also,
    \begin{align*}
        \att*{q}{\hat P}
        &=
        \frac{\int_{\kvspace} \kappa(q,k) v\,\diff \hat P(k,v)}{m_{\hat P}(q)}
        =
        \frac{\int_{\kvspace} a_k(q \mid P) v\,\diff \hat P(k,v)}{1+t(q)} \\
        &=
        \frac{\att*{q}{P} + U(q) + t(q)\att*{q}{P}}{1+t(q)}
        =
        \att*{q}{P} + \frac{U(q)}{1+t(q)}.
    \end{align*}
    Hence
    \[
        \att*{q}{\hat P} - \att*{q}{P} = \frac{U(q)}{1+t(q)}.
    \]
    On the other hand,
    \[
        \int_{\kvspace} \Gamma_P(k,v)(q)\,\diff \hat P(k,v)
        =
        \bigl(U(q),V t(q)\bigr)
        \in \R^{d_v+1}.
    \]

    Since both \(\att*{q}{P}\) and \(\att*{q}{\hat P}\) are convex combinations of values, we have
    \[
        \left\|\att*{q}{\hat P}-\att*{q}{P}\right\|_2 \le 2V.
    \]
    If \(1+t(q)\ge \frac12\), then
    \[
        \left\|\att*{q}{\hat P}-\att*{q}{P}\right\|_2^2
        =
        \left\|\frac{U(q)}{1+t(q)}\right\|_2^2
        \le 4\|U(q)\|_2^2.
    \]
    If \(1+t(q)<\frac12\), then \(t(q)^2\ge \frac14\), and therefore
    \[
        \left\|\att*{q}{\hat P}-\att*{q}{P}\right\|_2^2
        \le 4V^2
        \le 16V^2 t(q)^2.
    \]
    Thus, in all cases,
    \[
        \left\|\att*{q}{\hat P}-\att*{q}{P}\right\|_2^2
        \le
        16\bigl(\|U(q)\|_2^2 + V^2 t(q)^2\bigr)
        =
        16\left\|
            \int_{\kvspace} \Gamma_P(k,v)(q)\,\diff \hat P(k,v)
        \right\|_{\R^{d_v} \oplus \R}^2.
    \]
    Averaging over \(q\sim \nu\) then yields the claim.
\end{proof}

\begin{proof}[Proof of \Cref{pro:reduction}]
    We first show that \(\mu_{P,\nu}\) is centered. For every \(q \in \quespace\),
    \begin{align*}
        &\phantom{=}
        \int_{\kvspace} \Gamma_P(k,v)(q)\,\diff P(k,v) \\
        &=
        \left(
            \int_{\kvspace} a_k(q \mid P)\bigl(v-\att*{q}{P}\bigr)\,\diff P(k,v),
            \;
            V\left(\int_{\kvspace} a_k(q \mid P)\,\diff P(k,v)-1\right)
        \right) \\
        &=
        (0,0),
    \end{align*}
    since
    \[
        \int_{\kvspace} a_k(q \mid P)\,\diff P(k,v)=1
        \qquad\text{and}\qquad
        \int_{\kvspace} a_k(q \mid P)\,v\,\diff P(k,v)=\att*{q}{P}.
    \]
    Thus, \(\int_{\hilbert_\nu} g \diff \mu_{P,\nu}(g) = 0\).

    For the covariance identity, we have
    \begin{align*}
        \int_{\hilbert_\nu} g \otimes g \,\diff \mu_{P,\nu}(g)
        =
        \int_{\kvspace} \Gamma_P(k,v) \otimes \Gamma_P(k,v) \,\diff P(k,v)
        =
        \cov_{P,\nu}
    \end{align*}
    by the definition of \(\cov_{P,\nu}\).

    Now, let \(\hat\mu \in \Prob_K(\supp(\mu_{P,\nu}))\). Write
    \[
        \hat\mu = \sum_{\ell=1}^m \beta_\ell \delta_{g_\ell},
        \qquad
        m \le K,
    \]
    with \(g_\ell \in \supp(\mu_{P,\nu})\), \(\beta_\ell \ge 0\), and \(\sum_{\ell=1}^m \beta_\ell = 1\).

    For each \(\ell \in [m]\), choose one atom \((\hat k_\ell,\hat v_\ell) \in \supp(P)\) such that
    \[
        \Gamma_P(\hat k_\ell,\hat v_\ell) = g_\ell.
    \]
    This is possible because \(g_\ell \in \supp(P \circ \Gamma_P^{-1})\). Define
    \[
        \hat P \coloneqq \sum_{\ell=1}^m \beta_\ell \delta_{(\hat k_\ell,\hat v_\ell)}.
    \]
    Then, \(\hat P \in \Prob_K(\supp(P))\) depends on \(\hat\mu\) but not on \(\nu\), and
    \[
        \int_{\kvspace} \Gamma_P(k,v)\,\diff \hat P(k,v)
        =
        \sum_{\ell=1}^m \beta_\ell \Gamma_P(\hat k_\ell,\hat v_\ell)
        =
        \sum_{\ell=1}^m \beta_\ell g_\ell
        =
        \int_{\hilbert_\nu} g\,\diff \hat\mu(g).
    \]

    Therefore, by \Cref{lem:error-hilbert-transfer},
    \[
        \err*{P}{\nu}{\hat P}
        \le
        16\left\|
            \int_{\kvspace} \Gamma_P(k,v)\,\diff \hat P(k,v)
        \right\|_{\hilbert_\nu}^2
        =
        16\left\|
            \int_{\hilbert_\nu} g\,\diff \hat\mu(g)
        \right\|_{\hilbert_\nu}^2.
    \]
    Since \(\hat\mu \in \Prob_K(\supp(\mu_{P,\nu}))\) was arbitrary, that proves the claim.
\end{proof}

\subsubsection{Query-aware upper bound}
In the query-aware regime, the compressor may depend on \(\nu\), and therefore on the geometry of \(\hilbert_\nu\). After the reduction above, the problem is exactly to approximate the zero barycenter of the pushed-forward measure \(\mu_{P,\nu}\) by a \(K\)-atomic measure on \(\supp(\mu_{P,\nu})\).
This is exactly the setting of \Cref{thm:sparse-hilbert-approximation-fixed-K}, which provides a barycenter bound in terms of the spectral tail of its covariance, which here is \(\cov_{P,\nu}\). 
Pulling the resulting sparse measure back to \(\supp(P)\) then gives the desired compression bound for KV compression.

\thmsparsehilbertapproximationfixedk*
\begin{proof}
    Choose \(r \coloneqq \lfloor K/3 \rfloor\) and \(M \coloneqq K - 2r - 2\). Further, choose \(U\) to be the span of the top \(r_\ast \coloneqq \min\{r,\rank(\cov)\}\) eigenvectors of \(\cov\). Then, \(\dim(U)=r_\ast\le r\). Invoking \Cref{thm:sparse-hilbert-approximation} with subspace \(U\) and support size parameter \(M\) yields \(\hat\mu \in \Prob(\supp(\mu))\) with \(\supp(\hat\mu) \le M + 2r + 2 \le K\) and
    \[
        \left\|\int_\hilbert x\, d\hat\mu(x)\right\|_\hilbert^2
        \le
        \frac{20}{9M}\trace\bigl((I - \Pi_U) \cov\bigr)
        =
        \frac{20}{9M}\tail_r(\cov).
        \tag*{\qedhere}
    \]
\end{proof}
\corupperaware*
\begin{proof}
    This is immediate from \Cref{pro:reduction} and \Cref{thm:sparse-hilbert-approximation-fixed-K}.
\end{proof}

\subsubsection{Query-agnostic upper bound}
In the query-agnostic regime, the compressor cannot adapt to \(\nu\), so the sparse approximation argument from the aware case does not apply.
A natural query-independent alternative is the empirical measure of \(K\) i.i.d. samples from \(P\). Under the Hilbert embedding, this corresponds to averaging \(K\) independent copies of the centered random response profile \(G=\Gamma_P(X)\).
The mean remains zero, and the mean squared norm of the empirical average is of order \(\trace(\cov_{P,\nu})/K\). Using \Cref{lem:error-hilbert-transfer} from the reduction, this yields the trace upper bound.

\begin{lemma} \label{lem:monte-carlo-barycenter}
    Let \(\hilbert\) be a real Hilbert space, and let \(\mu \in \Prob_\fin(\hilbert)\) be centered, with covariance operator \(\Sigma\). 
    Let \(G_1,\dots,G_K \stackrel{\mathrm{iid}}{\sim} \mu\), and define
    \[
        \hat\mu_K \coloneqq \frac1K \sum_{i=1}^K \delta_{G_i}.
    \]
    Then,
    \[
        \Ex\left[
            \widetilde\Err_\mu(\hat\mu_K)
        \right]
        =
        \frac{\trace(\Sigma)}{K}.
    \]
\end{lemma}
\begin{proof}
    Since \(\mu\) is centered, the \(G_i\) are independent and mean zero. Hence
    \[
        \Ex\left[
            \widetilde\Err_\mu(\hat\mu_K)
        \right]
        =
        \Ex\left\|
            \frac1K \sum_{i=1}^K G_i
        \right\|_\hilbert^2
        =
        \frac1{K^2}\sum_{i=1}^K \Ex\|G_i\|_\hilbert^2
        =
        \frac{\trace(\Sigma)}{K}.
    \]
\end{proof}

\thmupperagnostic*
\begin{proof}
    Let \(P \in \kvmeasures\) and \(\nu \in \quemeasures\).
    Let \(X_1,\dots,X_K\stackrel{\mathrm{iid}}{\sim}P\), set
    \[
        G_i=\Gamma_P(X_i),
        \qquad
        \hat\mu_K=\frac1K\sum_{i=1}^K\delta_{G_i}.
    \]
    By \Cref{pro:reduction}, \(\hat\mu_K\) can be pulled back to the empirical context measure
    \[
        \hat P=\frac1K\sum_{i=1}^K\delta_{X_i}
    \]
    with
    \[
        \err*{P}{\nu}{\hat P}
        \lesssim
        \widetilde\Err_{\mu_{P,\nu}}(\hat\mu_K).
    \]
    Taking expectations and applying \Cref{lem:monte-carlo-barycenter} with \(\mu=\mu_{P,\nu}\) and \(\Sigma=\cov_{P,\nu}\) gives
    \[
        \Ex[\err*{P}{\nu}{\hat P}]
        \lesssim
        \frac{\trace(\cov_{P,\nu})}{K}.
    \]
    The claim then follows because random sampling from \(P\) can be implemented by a randomized query-agnostic compressor.
\end{proof}

\subsection{Lower bounds} \label{apx:lower-bounds}

We now provide lower bounds on the query-aware and query-agnostic minimax compression risk matching the upper bounds from \Cref{apx:upper-bounds}.
Both bounds are derived from a shared, explicit construction of an instance of a context \(P\) and a query distribution \(\nu\) that are realizable from actual keys, queries, and values under \Cref{asp:richness}.
This instance is constructed so that any significant compression of the context incurs large error over the chosen query distribution.
We then specialize this construction to the query-aware and query-agnostic regime by calibrating the error incurred under compression against the spectrum of \(\cov_{P,\nu}\), obtaining lower bounds that faithfully match their corresponding upper bounds throughout the entire claimed range of redundancy levels captured by \(\trace(\cov_{P,\nu})\) and \(\tail_{K/3}(\cov_{P,\nu})\), respectively.

Throughout this section, fix universal constants
\[
    \textstyle
    \rho_0\coloneqq \frac{99}{100}
    \qquad \text{and} \qquad
    \gamma_0\coloneqq \frac{9}{10}.
\]

\subsubsection{Hard instance construction}

We begin by constructing a pair \(P,\nu\) that is impossible to compress significantly without incurring large error.
The main idea is to use keys that can be sharply distinguished from each other, and have them carry almost orthogonal values.
We call such an instance a \emph{lookup instance} because querying a key of the context sharply retrieves the associated value.
\Cref{def:lookup-instance} formalizes the properties of a lookup instance formally, and \Cref{lem:lookup-instance-membership} shows that such instances can indeed be constructed from actual keys, queries, and values under \Cref{asp:richness}.
In particular, \Cref{asp:richness} provides enough ambient dimension to choose well-separated keys and values from Hamming packings, and enough key/query norm budget to address the keys sharply through softmax attention.

\begin{definition}[Lookup instance] \label{def:lookup-instance}
    Fix an even integer \(N \ge 2\) and let
    \[
        \textstyle
        P_N \coloneqq \frac{1}{N} \sum_{i=1}^N \delta_{(k_i,v_i)},
        \qquad
        \nu_N \coloneqq \frac{1}{N} \sum_{i=1}^N \delta_{k_i}.
    \]
    We call the pair \((P_N,\nu_N)\) a \emph{lookup instance of size \(N\)} if
    \begin{enumerate}[
        labelindent=1.5em,
        labelwidth=\widthof{\textbf{(membership)}},
        leftmargin=!,
        labelsep=0.6em,
        align=left,
        font=\bfseries
    ]
        \item[(membership)] \(k_i\in\keyspace \cap \quespace\), and \(v_i\in\valspace\) for all \(i\in[N]\);
        \item[(centering)]
        \(
            \frac{1}{N} \sum_{i=1}^N v_i = 0;
        \)
        \item[(boundary)]
        \(
            \|v_i\|_2=\rho_0 V
        \)
        for all \(i\in[N]\);
        \item[(separation)]
        \(
            \dotp{v_i}{v_j}
            \le
            \gamma_0\rho_0^2V^2
        \)
        for all \(i \neq j\);
        \item[(sharpness)]
        \(
            \frac{1}{N} \sum_{j\neq i} a_{k_j}(k_i\mid P_N)
            \le
            N^{-3/2}
        \)
        for all \(i\in[N]\).
    \end{enumerate}
\end{definition}

\begin{lemma} \label{lem:lookup-instance-membership}
    Let \(K\ge 50\), set \(N\coloneqq 2K\), and suppose \Cref{asp:richness} holds.
    Then, there exists a lookup instance
    \(
        \bigl(P_N,\nu_N\bigr)
        \in
        \kvmeasures*[N] \times \quemeasures
    \)
    of size \(N\) such that
    \begin{enumerate}[label=(\roman*)]
        \item \(|\supp(\nu_N)| = |\supp(P_N)| = N\);
        \item \(\dim \vspan\{v_1,\ldots,v_N\} \le K/6\); and
        \item \(\|k\|_2 = 5d_k^{1/4}\sqrt{\log N}\) for all \(k \in \supp(\nu_N)\).
    \end{enumerate}
\end{lemma}
\begin{proof}
    For \(r\ge 1\), set
    \[
        B_r \coloneqq \sum_{\ell=0}^{\lfloor r/20\rfloor}\binom{r}{\ell}.
    \]
    We first construct a Hamming packing.
    Greedily select words from \(x \in \{\pm1\}^r\).
    After selecting \(x\), delete the Hamming balls of radius \(\lfloor r/20\rfloor\) around \(x\) and around \(-x\).
    Each step deletes at most \(2B_r\) words, so the set \(\mathsf C_r\subseteq\{\pm1\}^r\) selected by the greedy procedure satisfies
    \[
        |\mathsf C_r|
        \ge
        \frac{2^r}{2B_r}.
    \]
    Denote by \(d_H\) the Hamming distance.
    If \(x,x'\in\mathsf C_r\) are distinct, then
    \[
        d_H(x,x')>\lfloor r/20\rfloor
        \qquad
        \text{and}
        \qquad
        d_H(x,-x')>\lfloor r/20\rfloor.
    \]
    Since \(\dotp{x}{x'}=r-2d_H(x,x')\), this implies
    \[
        |\dotp{x}{x'}|
        \le
        \gamma_0 r .
    \]

    Using the entropy bound for Hamming balls, for \(r\ge 20\),
    \[
        B_r
        \le
        \exp\!\left(
            r\left[
                -\frac{1}{20}\log\frac{1}{20}
                -\frac{19}{20}\log\frac{19}{20}
            \right]
        \right)
        \le
        \exp(0.2r).
    \]
    Hence, for \(r\ge 20\),
    \begin{equation} \label{eq:lookup-instance-membership:code-size}
        |\mathsf C_r|
        \ge
        \frac{1}{2}\exp((\log 2-0.2)r).
    \end{equation}
    In particular, if \(r\ge 10\log K\), then
    \begin{equation} \label{eq:lookup-instance-membership:code-size-large}
        |\mathsf C_r|
        \ge
        \frac{1}{2}K^{4.9}
        \ge
        2K,
    \end{equation}
    where the last inequality uses \(K\ge 50\).
    Also, if \(r\ge 8\), then
    \begin{equation} \label{eq:lookup-instance-membership:code-size-small}
        |\mathsf C_r|\ge 6r+5.
    \end{equation}
    Indeed, for \(8\le r<20\), this follows from \(B_r=1\) and \(2^{r-1}\ge 6r+5\).
    For \(r\ge20\), \Cref{eq:lookup-instance-membership:code-size} gives
    \[
        |\mathsf C_r|
        \ge
        \frac{1}{2}\exp((\log 2-0.2)r)
        \ge
        6r+5.
    \]

    Let \(N\coloneqq 2K\).
    Choose distinct \(\xi_1,\ldots,\xi_N\in\mathsf C_{d_k}\), and define
    \[
        R_k\coloneqq 5d_k^{1/4}\sqrt{\log N},
        \qquad \text{and} \qquad
        k_i\coloneqq R_k\frac{\xi_i}{\sqrt{d_k}}
    \]
    for \(i\in[N]\).
    Since \(d_k\ge 10\log K\), the code \(\mathsf C_{d_k}\) has at least \(2K=N\) words.
    Moreover,
    \[
        \|k_i\|_2=R_k
        =
        5d_k^{1/4}\sqrt{\log N}
        \le
        8d_k^{1/4}\sqrt{\log K},
    \]
    so \(k_i\in\keyspace\cap\quespace\) by \Cref{asp:richness}.

    Set
    \[
        r_v\coloneqq \min\{d_v,\lfloor K/6\rfloor\}.
    \]
    Then \(r_v\ge 8\), since \(d_v\ge 10\log K\) and \(K\ge 50\).
    The code \(\mathsf C_{r_v}\) has at least \(K\) words.
    If \(r_v=d_v\), this follows from \Cref{eq:lookup-instance-membership:code-size-large}, because then \(r_v\ge 10\log K\).
    If \(r_v=\lfloor K/6\rfloor\), it follows from \Cref{eq:lookup-instance-membership:code-size-small}, because then \(r_v\ge 8\) and \(6r_v+5\ge K\).
    Choose distinct \(\zeta_1,\ldots,\zeta_K\in\mathsf C_{r_v}\), and let \(E:\R^{r_v}\to\R^{d_v}\) be the canonical coordinate embedding.
    Define
    \[
        u_\ell\coloneqq E\frac{\zeta_\ell}{\sqrt{r_v}},
        \qquad
        v_\ell\coloneqq \rho_0V u_\ell,
        \qquad
        v_{K+\ell}\coloneqq -\rho_0V u_\ell
    \]
    for \(\ell \in[K]\).
    Then, \(v_i\in\valspace=\ball_{d_v}(V)\),
    \[
        \frac{1}{N}\sum_{i=1}^N v_i=0,
        \qquad \text{and} \qquad
        \|v_i\|_2=\rho_0V
    \]
    for all \(i\in[N]\).
    For \(i\neq j\),
    \[
        \dotp{v_i}{v_j}
        \le
        \gamma_0\rho_0^2V^2.
    \]
    Also,
    \[
        \dim\vspan\{v_1,\ldots,v_N\}
        \le
        \dim \im(E)
        \le
        r_v
        \le
        K/6.
    \]

    Define
    \[
        P_N\coloneqq \frac{1}{N}\sum_{i=1}^N\delta_{(k_i,v_i)},
        \qquad \text{and} \qquad
        \nu_N\coloneqq \frac{1}{N}\sum_{i=1}^N\delta_{k_i}.
    \]
    Since \(k_i \in \keyspace \cap \quespace\) and \(\|v_i\|_2 \le V\) for all \(i \in [N]\), \(P_N \in \kvmeasures*[N]\) and \(\nu_N \in \quemeasures\).
    
    The keys \(k_i\) are distinct, so
    \[
        |\supp(P_N)|=|\supp(\nu_N)|=N.
    \]

    It remains to verify sharpness.
    For \(i\neq j\),
    \[
        \frac{\dotp{k_i}{k_j}}{\sqrt{d_k}}
        \le
        25\gamma_0\log N,
    \]
    while
    \[
        \frac{\|k_i\|_2^2}{\sqrt{d_k}}
        =
        25\log N.
    \]
    Hence, for \(j\neq i\),
    \[
        a_{k_j}(k_i\mid P_N)
        =
        \frac{
            N\exp(\dotp{k_i}{k_j}/\sqrt{d_k})
        }{
            \sum_{m=1}^N\exp(\dotp{k_i}{k_m}/\sqrt{d_k})
        }
        \le
        N\exp(-25(1-\gamma_0)\log N)
        =
        N^{-3/2}.
    \]
    Therefore,
    \[
        \frac{1}{N}\sum_{j\neq i}a_{k_j}(k_i\mid P_N)
        \le
        N^{-3/2}
    \]
    for all \(i\in[N]\).
    
    Thus, \((P_N,\nu_N)\) is a lookup instance of size \(N\), and it satisfies \textup{(i)}--\textup{(iii)}.
\end{proof}

\subsubsection{Approximating lookup instances requires almost the full context}

We now show that lookup instances are hard to compress.
Indeed, querying a key of such an instance retrieves the associated value almost perfectly, and the retrievable values are separated by a constant fraction of the value norm.
As a consequence, compressing lookup instances to even half their context size incurs compression error at the maximum possible scale \(V^2\).

\begin{lemma} \label{lem:lookup-output-geometry}
    Fix an even integer \(N \ge 100\), and let \((P_N,\nu_N)\) be the lookup instance of size \(N\) from \Cref{lem:lookup-instance-membership}.
    Then, \(\att{0}{P_N}=0\), and, for every \((k,v) \in \supp(P_N)\),
    \[
        \|\att{k}{P_N} - v\|_2
        \le
        2\rho_0 V N^{-3/2}.
    \]
    Consequently,
    \[
        \|\att{k}{P_N}\|_2
        \le
        \frac{249}{250}V
        \qquad \text{and} \quad
        \|\att{k}{P_N} - \att{k'}{P_N}\|_2
        \ge
        \frac{V}{4}
    \]
    for all \(k,k' \in \supp(\nu_N)\) with \(k \neq k'\).
\end{lemma}

\begin{lemma} \label{lem:hard-query-lower-bound}
    Fix an even integer \(N\ge 100\), and let \((P_N,\nu_N)\) be the lookup instance of size \(N\) from \Cref{lem:lookup-instance-membership}.
    Assume \(\|v\|_2 \le V\) for all \(v \in \valspace\).
    Then, for every \(\hat P\in\kvmeasures[N/2]\),
    \[
        \err*{P_N}{\nu_N}{\hat P}
        \gtrsim
        V^2.
    \]
\end{lemma}

\begin{proof}[Proof of \Cref{lem:lookup-output-geometry}]
    Write
    \[
        P_N = \frac{1}{N} \sum_{i=1}^N \delta_{(k_i,v_i)},
        \qquad
        \nu_N = \frac{1}{N} \sum_{i=1}^N \delta_{k_i},
    \]
    as in \Cref{def:lookup-instance}, and set
    \[
        w_j(q) \coloneqq \frac{1}{N} a_{k_j}(q\mid P_N).
    \]
    Then, \(w_j(q)\ge 0\), \(\sum_{j=1}^N w_j(q)=1\), and
    \[
        \att{q}{P_N} = \sum_{j=1}^N w_j(q)v_j .
    \]

    Since \(a_{k_j}(0 \mid P_N) = 1\) for every \(j \in [N]\), centering gives
    \[
        \att{0}{P_N}
        =
        \frac{1}{N} \sum_{j=1}^N v_j
        =
        0 .
    \]

    Fix \(i\in[N]\). By sharpness,
    \[
        \sum_{j\neq i} w_j(k_i)
        =
        \frac{1}{N} \sum_{j\neq i} a_{k_j}(k_i\mid P_N)
        \le
        N^{-3/2}.
    \]
    Hence, using \(\|v_j\|_2=\rho_0V\) for all \(j \in [N]\),
    \begin{align*}
        \|\att{k_i}{P_N}-v_i\|_2
        &=
        \left\|
            (w_i(k_i)-1)v_i
            +
            \sum_{j\neq i} w_j(k_i)v_j
        \right\|_2
        \\&\le
        (1-w_i(k_i))\|v_i\|_2
        +
        \sum_{j\neq i}w_j(k_i)\|v_j\|_2
        \\&=
        2\rho_0V \sum_{j\neq i} w_j(k_i)
        \le
        2\rho_0VN^{-3/2}.
    \end{align*}
    This proves the approximation bound.

    The norm bound follows immediately. For \(N\ge100\),
    \[
        \|\att{k_i}{P_N}\|_2
        \le
        \|v_i + (\att{k_i}{P_N} - v_i)\|_2
        \le
        \|v_i\|_2+2\rho_0VN^{-3/2}
        =
        \rho_0V(1+2N^{-3/2})
        \le
        \frac{249}{250}V .
    \]
    Finally, if \(i\neq j\), then the separation condition gives
    \[
        \|v_i - v_j\|_2^2
        =
        2\rho_0^2 V^2 - 2 \dotp{v_i}{v_j}
        \ge
        2(1-\gamma_0)\rho_0^2V^2.
    \]
    Therefore, again using \(N\ge100\), \(\rho_0=99/100\), and \(\gamma_0=9/10\),
    \begin{align*}
        \|\att{k_i}{P_N} - \att{k_j}{P_N}\|_2
        &\ge
        \|v_i - v_j\|_2
        -
        \|\att{k_i}{P_N} - v_i\|_2
        -
        \|\att{k_j}{P_N} - v_j\|_2
        \\&\ge
        \rho_0V\sqrt{2(1-\gamma_0)}
        -
        4\rho_0VN^{-3/2}
        \ge
        \frac{V}{4}.
    \end{align*}
    That proves the claim.
\end{proof}

\begin{proof}[Proof of \Cref{lem:hard-query-lower-bound}]
    Write
    \[
        P_N = \frac{1}{N} \sum_{i=1}^N \delta_{(k_i,v_i)},
        \qquad
        \nu_N = \frac{1}{N} \sum_{i=1}^N \delta_{k_i},
    \]
    and let
    \[
        y_i\coloneqq \att{k_i}{P_N}.
    \]
    By \Cref{lem:lookup-output-geometry},
    \[
        \|y_i-v_i\|_2\le 2\rho_0VN^{-3/2}
    \]
    for all \(i\in[N]\).

    Fix \(\hat P\in\kvmeasures[N/2]\), and write
    \[
        \hat P=\sum_{r=1}^M \beta_r\delta_{(\hat k_r,\hat v_r)}
    \]
    for \(M\le N/2\).
    By assumption, \(\|\hat v_r\|_2\le V\) for all \(r \in [M]\).
    For each \(i\in[N]\), set
    \[
        z_i\coloneqq \att{k_i}{\hat P}.
    \]
    Since \(z_i\) is a convex combination of the values \(\hat v_1,\ldots,\hat v_M\), we may write
    \[
        z_i=\sum_{r=1}^M \omega_{ir}\hat v_r,
    \]
    with \(\omega_{ir} \ge 0\) and \(\|\omega_i\|_1=1\) for all \(i \in [N]\).

    Let
    \[
        \mathcal I
        \coloneqq
        \left\{
            i\in[N] \mid \|z_i-y_i\|_2\le \frac{V}{200}
        \right\}
    \]
    be the well-approximated indices.
    We claim that \(|\mathcal I|\le M\).
    To that end, fix \(i\in\mathcal I\). Then,
    \begin{align*}
        \dotp{z_i}{v_i}
        &=
        \dotp{y_i}{v_i} + \dotp{z_i - y_i}{v_i}
        \\&\ge
        \dotp{y_i}{v_i}
        -
        \|z_i - y_i\|_2 \|v_i\|_2
        \\&=
        \dotp{v_i}{v_i} + \dotp{y_i - v_i}{v_i}
        -
        \|z_i - y_i\|_2 \|v_i\|_2
        \\&\ge
        \|v_i\|_2^2
        -
        \|y_i-v_i\|_2 \|v_i\|_2
        -
        \frac{V}{200} \|v_i\|_2
        \\&\ge
        \left(\rho_0^2 - 2\rho_0^2N^{-3/2} - \frac{\rho_0}{200}\right)V^2
        \ge
        \frac{97}{100}V^2,
    \end{align*}
    where the last step uses \(N\ge100\) and \(\rho_0=99/100\).
    Since \(z_i=\sum_r\omega_{ir}\hat v_r\), there exists \(r(i) \in [M]\) such that
    \[
        \dotp{\hat v_{r(i)},v_i}
        \ge
        \frac{97}{100}V^2.
    \]

    The map \(i\mapsto r(i)\) is injective.
    Indeed, suppose that \(r(i)=r(j)=r\) for distinct \(i,j\in\mathcal I\).
    Set
    \[
        u\coloneqq \frac{\hat v_r}{V},
        \qquad
        s_i\coloneqq \frac{v_i}{\rho_0V},
        \qquad \text{and} \qquad
        s_j\coloneqq \frac{v_j}{\rho_0V}.
    \]
    Then, \(\|u\|_2 \le 1\), \(\|s_i\|_2 = \|s_j\|_2 = 1\), and by separation,
    \[
        \dotp{s_i}{s_j} \le \gamma_0.
    \]
    On the other hand,
    \[
        \dotp{u, s_i + s_j}
        \ge
        \frac{2\cdot 97}{100\rho_0}
        =
        \frac{194}{99}.
    \]
    This contradicts Cauchy--Schwarz, since
    \[
        \dotp{u,s_i+s_j}
        \le
        \|s_i+s_j\|_2
        \le
        \sqrt{2+2\gamma_0}
        =
        \sqrt{\frac{19}{5}}
        <
        \frac{194}{99}.
    \]
    Hence, \(|\mathcal I|\le M\le N/2\).

    Therefore, at least \(N/2\) indices satisfy \(\|z_i - y_i\|_2>V/200\).
    Since \(\nu_N\) is uniform on the hard queries,
    \begin{align*}
        \err*{P_N}{\nu_N}{\hat P}
        &=
        \frac{1}{N} \sum_{i=1}^N \|y_i-z_i\|_2^2
        \\&\ge
        \frac{1}{N} \cdot \frac{N}{2} \cdot \left(\frac{V}{200}\right)^2
        =
        \frac{V^2}{80000}.
    \end{align*}
    That proves the claim.
\end{proof}

\subsubsection{Query-aware lower bound}

With the lookup instance established as hard to compress, we turn to the query-aware minimax lower bound.
By itself, the lookup instance only gives the lower bound at the lowest redundancy scale, where \(\tail_{K/3}(\cov_{P,\nu})\) is of order \(V^2K\) and every \(K\)-atomic approximation incurs error of order \(V^2\).
To show that \(\tail_{K/3}(\cov_{P,\nu})\) faithfully governs the minimax risk throughout all scales \(R\in[0,V^2K]\), we need the same obstruction at every intermediate such scale.
We achieve this by diluting the lookup query distribution with the easy query distribution \(\delta_0\).
The dummy query contributes only low-rank covariance, while the hard queries contribute both the \((K/3)\)-tail and the approximation error.
Thus, the mixture \(p\nu_N+(1-p)\delta_0\) scales the relevant redundancy and the compression difficulty by the same factor \(p\), allowing us to tune both to the desired level \(R\).

\begin{lemma} \label{lem:lookup-diluted-covariance}
    Let \(K\ge 50\) and set \(N\coloneqq 2K\).
    Suppose \Cref{asp:richness} holds, and let \((P_N,\nu_N)\) be the lookup instance from \Cref{lem:lookup-instance-membership}.
    For \(p\in[0,1]\), define the diluted query distribution
    \[
        \nu_{N,p}
        \coloneqq
        p \, \nu_N + (1-p) \, \delta_0.
    \]
    Then,
    \[
        \tail_{K/3}(\cov_{P_N,\nu_{N,p}})
        \asymp
        p \, V^2K.
    \]
\end{lemma}

\begin{lemma} \label{lem:two-query-fitting-lookup}
    Let \(K\ge 50\) and set \(N\coloneqq 2K\). 
    Suppose \Cref{asp:richness} holds and let \((P_N,\nu_N)\) be the lookup instance from \Cref{lem:lookup-instance-membership}.
    Then, for every \(k \in \supp(\nu_N)\), there exists \(\hat P_k\in\kvmeasures[2]\) such that
    \[
        \att{0}{\hat P_k}
        =
        \att{0}{P_N},
        \qquad \text{and} \qquad
        \att{k}{\hat P_k}
        =
        \att{k}{P_N}.
    \]
    Consequently, for every \(p\in[0,1]\) and every \(k \in \supp(\nu_N)\),
    \[
        \err*{P_N}{p \, \delta_{k} + (1-p) \, \delta_0}{\hat P_k}
        =
        0.
    \]
\end{lemma}

\thmlowerboundaware*

\begin{proof}[Proof of \Cref{lem:lookup-diluted-covariance}]
    Write
    \[
        P_N = \frac{1}{N} \sum_{i=1}^N \delta_{(k_i,v_i)},
        \qquad \text{and} \qquad
        \nu_N = \frac{1}{N} \sum_{i=1}^N \delta_{k_i},
    \]
    and abbreviate
    \[
        a_{ij} \coloneqq a_{k_j}(k_i\mid P_N),
        \qquad \text{and} \qquad
        y_i \coloneqq \att{k_i}{P_N}.
    \]
    Let
    \[
        \mathsf H_0
        \coloneqq
        \left\{
            u\in L_2(P_N) \mid \int_\kvspace u \, \diff P_N=0
        \right\}.
    \]
    This space has dimension \(N-1\).
    Define
    \[
        T_\rho: \mathsf H_0\to \hilbert_\rho,
        \qquad
        u
        \mapsto
        \int_\kvspace \Gamma_{P_N}(z)u(z)\,\diff P_N(z)
    \]
    and let
    \[
        G_\rho \coloneqq T_\rho^*T_\rho
    \]
    be the corresponding Gram operator on \(\mathsf H_0\).
    Since
    \[
        \int_\kvspace \Gamma_{P_N}\,\diff P_N=0,
    \]
    we have
    \[
        T_\rho T_\rho^*=\cov_{P_N,\rho}.
    \]
    Hence, \(G_\rho\) and \(\cov_{P_N,\rho}\) have the same nonzero eigenvalues.

    We first prove
    \[
        \tail_{K/3}(G_{\nu_N}) \asymp V^2K.
    \]
    For \(u\in\mathsf H_0\), write \(u_i\coloneqq u(k_i,v_i)\), so that
    \[
        \frac{1}{N} \sum_{i=1}^N u_i=0,
        \qquad \text{and} \qquad
        \|u\|_{L_2(P_N)}^2=\frac{1}{N} \sum_{i=1}^N u_i^2.
    \]
    The scalar component of \(T_{\nu_N}u\) at query \(k_i\) is
    \[
        (T_{\nu_N}u)_{\mathrm{sc}}(k_i)
        \coloneqq
        \frac{V}{N}\sum_{j=1}^N(a_{ij}-1)u_j .
    \]
    By sharpness,
    \[
        s_i \coloneqq \sum_{j\neq i}a_{ij}\le N^{-1/2},
        \qquad \text{and} \qquad
        a_{ii}=N-s_i.
    \]
    Hence, for \(u\in\mathsf H_0\),
    \begin{align*}
        (T_{\nu_N}u)_{\mathrm{sc}}(k_i)
        &=
        \frac{V}{N}\sum_{j=1}^N(a_{ij}-1)u_j
        \\&=
        Vu_i + \frac{V}{N}\left(
            -s_i u_i+\sum_{j\neq i}a_{ij}u_j
        \right).
    \end{align*}
    If \(\|u\|_{L_2(P_N)}=1\), then \(|u_j|\le \sqrt N\) for all \(j\), and therefore
    \[
        \left|
            (T_{\nu_N}u)_{\mathrm{sc}}(k_i)-Vu_i
        \right|
        \le
        \frac{2V}{N}.
    \]
    Consequently,
    \begin{align*}
        \|T_{\nu_N}u\|_{\hilbert_{\nu_N}}
        &\ge
        \left(
            \frac1N\sum_{i=1}^N
            \left|(T_{\nu_N}u)_{\mathrm{sc}}(k_i)\right|^2
        \right)^{1/2}
        \\
        &\ge
        \left(
            \frac1N\sum_{i=1}^N V^2u_i^2
        \right)^{1/2}
        -
        \left(
            \frac1N\sum_{i=1}^N \frac{4V^2}{N^2}
        \right)^{1/2}        \\
        &\ge
        V-\frac{2V}{N}
        \ge
        \frac{V}{2}.
    \end{align*}
    Thus,
    \[
        G_{\nu_N}\succeq \frac{V^2}{4}I_{\mathsf H_0}.
    \]
    Since \(\dim \mathsf H_0=N-1=2K-1\),
    \[
        \tail_{K/3}(G_{\nu_N})
        \gtrsim
        V^2K.
    \]

    For the reverse bound, again use \(s_i\le N^{-1/2}\). Then,
    \[
        \sum_{j=1}^N a_{ij}^2
        =
        (N-s_i)^2+\sum_{j\neq i}a_{ij}^2
        \le
        N^2+1.
    \]
    Since \(\sum_j a_{ij}=N\),
    \[
        \sum_{j=1}^N(a_{ij}-1)^2
        =
        \sum_{j=1}^Na_{ij}^2-N
        \le
        N^2+1.
    \]
    By \Cref{lem:lookup-output-geometry}, \(\|y_i\|_2\le V\), and by the boundary condition, \(\|v_j\|_2\le V\). Hence,
    \[
        \|v_j-y_i\|_2\le 2V.
    \]
    Therefore, using the trace formula for \(\cov_{P_N,\delta_{k_i}}\),
    \begin{align*}
        \trace(\cov_{P_N,\delta_{k_i}})
        &=
        \frac1N\sum_{j=1}^N
        \left[
            V^2(a_{ij}-1)^2
            +
            a_{ij}^2\|v_j-y_i\|_2^2
        \right]        \\
        &\le
        \frac1N
        \left[
            V^2(N^2+1)
            +
            4V^2(N^2+1)
        \right]
        \lesssim
        V^2N.
    \end{align*}
    Averaging over \(i\) gives
    \[
        \trace(G_{\nu_N})
        =
        \trace(\cov_{P_N,\nu_N})
        \lesssim
        V^2N
        \asymp
        V^2K,
    \]
    and hence
    \[
        \tail_{K/3}(G_{\nu_N})
        \lesssim
        V^2K.
    \]
    Combining the two bounds gives
    \[
        \tail_{K/3}(G_{\nu_N})
        \asymp
        V^2K.
    \]

    Set
    \[
        G_N\coloneqq G_{\nu_N},
        \qquad
        G_0\coloneqq G_{\delta_0},
        \qquad \text{and} \qquad
        G_p\coloneqq G_{\nu_{N,p}}.
    \]
    For every \(u\in\mathsf H_0\),
    \begin{align*}
        \dotp{u}{G_pu}_{\mathsf H_0}
        &=
        \|T_{\nu_{N,p}}u\|_{\hilbert_{\nu_{N,p}}}^2        \\
        &=
        p\|T_{\nu_N}u\|_{\hilbert_{\nu_N}}^2
        +(1-p)\|T_{\delta_0}u\|_{\hilbert_{\delta_0}}^2        \\
        &=
        \dotp{u}{(pG_N+(1-p)G_0)u}_{\mathsf H_0}.
    \end{align*}
    Hence,
    \[
        G_p=pG_N+(1-p)G_0
    \]
    as operators on \(\mathsf H_0\).

    At query \(0\), the scalar response vanishes and the value response lies in \(\vspan\{v_1,\ldots,v_N\}\).
    By \Cref{lem:lookup-instance-membership},
    \[
        \rank(G_0) \le \dim \vspan\{v_1,\ldots,v_N\}\le K/6.
    \]
    The lower bound follows from \(G_p\succeq pG_N\). By eigenvalue monotonicity,
    \[
        \tail_{K/3}(G_p)
        \ge p\,\tail_{K/3}(G_N)
        \gtrsim pV^2K.
    \]

    For the upper bound, let \(U\subset \mathsf H_0\) be any subspace of dimension at most \(K/3\) containing \(\ran(G_0)\). Then,
    \[
        \trace\bigl((I-\Pi_U)G_0\bigr)=0.
    \]
    Hence, by Ky Fan's variational formula,
    \begin{align*}
        \tail_{K/3}(G_p)
        &\le
        \trace\bigl((I-\Pi_U)G_p\bigr)        \\
        &=
        p\,\trace\bigl((I-\Pi_U)G_N\bigr)
        +(1-p)\trace\bigl((I-\Pi_U)G_0\bigr)        \\
        &=
        p\,\trace\bigl((I-\Pi_U)G_N\bigr)
        \le
        p\,\trace(G_N)
        \lesssim
        pV^2K.
    \end{align*}
    Thus,
    \[
        \tail_{K/3}(G_p)
        \asymp
        pV^2K.
    \]
    Since \(G_p\) and \(\cov_{P_N,\nu_{N,p}}\) have the same nonzero eigenvalues,
    \[
        \tail_{K/3}(\cov_{P_N,\nu_{N,p}})
        \asymp
        pV^2K.
        \tag*{\qedhere}
    \]
\end{proof}

\begin{proof}[Proof of \Cref{lem:two-query-fitting-lookup}]
    Fix \(k\in\supp(\nu_N)\), and set
    \[
        y \coloneqq \att{k}{P_N}.
    \]
    By \Cref{lem:lookup-instance-membership},
    \[
        \|k\|_2 = 5d_k^{1/4} \sqrt{\log N}.
    \]
    Since \(N=2K\) and \(K\ge 50\),
    \[
        \|k\|_2
        \le
        8d_k^{1/4}\sqrt{\log K},
    \]
    so \(k,-k\in\keyspace\) by \Cref{asp:richness}. Moreover,
    \[
        \frac{\|k\|_2^2}{\sqrt{d_k}}
        =
        25\log N.
    \]
    By \Cref{lem:lookup-output-geometry},
    \[
        \att{0}{P_N}=0,
        \qquad \text{and} \qquad
        \|y\|_2 \le \frac{249}{250}V.
    \]
    Since \(N\ge100\),
    \[
        \tanh(25\log N)
        =
        \frac{N^{50}-1}{N^{50}+1}
        \ge
        \frac{249}{250}.
    \]
    Hence the vector
    \[
        w \coloneqq \frac{y}{\tanh(25\log N)}
    \]
    satisfies \(\|w\|_2\le V\), and therefore \(w,-w\in\valspace\).

    Define
    \[
        \hat P_k
        \coloneqq
        \frac{1}{2} \delta_{(k,w)} + \frac{1}{2} \delta_{(-k,-w)}.
    \]
    At query \(0\), the two logits are equal, so
    \[
        \att{0}{\hat P_k}
        =
        \frac{1}{2} w + \frac{1}{2}(-w)
        =
        0
        =
        \att{0}{P_N}.
    \]
    At query \(k\), the two logits are
    \[
        \frac{\langle k,k\rangle}{\sqrt{d_k}}
        =
        25\log N,
        \qquad \text{and} \qquad
        \frac{\langle k,-k\rangle}{\sqrt{d_k}}
        =
        -25\log N.
    \]
    Thus
    \[
        \att{k}{\hat P_k}
        =
        \tanh(25\log N) \, w
        =
        y
        =
        \att{k}{P_N}.
    \]
    That proves the first claim.

    The final claim follows because the query distribution \(p\delta_k+(1-p)\delta_0\) is supported on the two points where \(P_N\) and \(\hat P_k\) have identical attention outputs, so
    \begin{align*}
        &\phantom{=}\err*{P_N}{p\delta_k+(1-p)\delta_0}{\hat P_k}
        \\&=
        p \|\att{k}{P_N} - \att{k}{\hat P_k}\|_2^2
        + (1-p) \|\att{0}{P_N} - \att{0}{\hat P_k}\|_2^2
        =
        0.
        \tag*{\qedhere}
    \end{align*}
\end{proof}

\begin{proof}[Proof of \Cref{thm:lower-bound-aware}]
    Let \(N \coloneqq 2K\), and let \((P_N,\nu_N)\) be the lookup instance from \Cref{lem:lookup-instance-membership}.
    Set
    \[
        p \coloneqq \frac{R}{V^2K},
        \qquad \text{and} \qquad
        \nu_R\coloneqq p\,\nu_N+(1-p)\,\delta_0,
    \]
    and define
    \[
        \mathcal F \coloneqq \{(P_N,\nu_R)\}.
    \]
    Since \(R\in[0,V^2K]\), we have \(p\in[0,1]\). Also, by \Cref{lem:lookup-instance-membership}, \(|\supp(P_N)|=N=2K\) and \(\mathcal F \subseteq \kvmeasures*[2K] \times \quemeasures\).

    By \Cref{lem:lookup-diluted-covariance},
    \[
        \tail_{K/3}(\cov_{P_N,\nu_R})
        \asymp
        pV^2K
        =
        R.
    \]

    Now, fix any query-aware \(K\)-atomic compressor \(A\), and let \(\hat P\) be its random output on \((P_N,\nu_R)\).
    For every realization of \(\hat P\),
    \[
        \err*{P_N}{\nu_R}{\hat P}
        =
        p \, \err*{P_N}{\nu_N}{\hat P}
        + (1-p) \, \err*{P_N}{\delta_0}{\hat P}
        \ge
        p \, \err*{P_N}{\nu_N}{\hat P}.
    \]
    By \Cref{lem:hard-query-lower-bound},
    \[
        \err*{P_N}{\nu_N}{\hat P}
        \gtrsim
        V^2
    \]
    almost surely. Hence,
    \[
        \err*{P_N}{\nu_R}{\hat P}
        \gtrsim
        pV^2
        =
        \frac{R}{K}
    \]
    almost surely.
    Taking expectation over the randomness of \(A\), and then taking the infimum over all query-aware \(K\)-atomic compressors, gives
    \[
        \minimax{\compspace{K}[aw]}{\mathcal F}
        \gtrsim
        \frac{R}{K}.
        \tag*{\qedhere}
    \]
\end{proof}

\subsubsection{Query-agnostic lower bound}

We now prove the query-agnostic lower bound and simultaneously show a separation from the query-aware regime.
We use the same hard lookup instance as before, but expose it through a family of query distributions of the form \(p\delta_k+(1-p)\delta_0\), one for each key \(k\in\supp(\nu_N)\).
Each member of this family is easy for query-aware compression: once the query law is known, the compressor only has to match the output at \(0\) and at one key.
At the same time, the family is hard for query-agnostic compression because a query-agnostic compressor must choose one \(K\)-atomic summary that protects against all active keys.
Averaging over the active key then recovers the lookup query distribution \(\nu_N\), on which any compression to half the context size necessarily incurs large error.

As in the query-aware regime, dilution by the \(0\) query controls the error and redundancy scale.
Each hard query contributes trace of order \(V^2K\), while the dummy query contributes trace of order \(V^2\).
Thus, the mixture \(p\delta_k+(1-p)\delta_0\) tunes \(\trace(\cov_{P,\nu})\) to the target level \(T\in[V^2,V^2K]\), while the averaged error retains the corresponding fraction of the lookup compression error.

\begin{lemma} \label{lem:lookup-pointwise-covariance}
    Let \(K\ge 50\) and set \(N\coloneqq 2K\).
    Suppose \Cref{asp:richness} holds and let \((P_N,\nu_N)\) be the lookup instance from \Cref{lem:lookup-instance-membership}.
    Then,
    \[
        \trace(\cov_{P_N,\delta_0})
        \asymp
        V^2,
    \]
    and, for every \(k \in \supp(\nu_N)\),
    \[
        \trace(\cov_{P_N,\delta_{k}})
        \asymp
        V^2K.
    \]
    Consequently, for every \(p\in[0,1]\) and every \(k \in \supp(\nu_N)\),
    \[
        \trace(\cov_{P_N, p \, \delta_{k} + (1-p) \, \delta_0})
        \asymp
        pV^2K+V^2.
    \]
\end{lemma}

\thmlowerboundagnostic*

\begin{proof}[Proof of \Cref{lem:lookup-pointwise-covariance}]
    Write
    \[
        P_N = \frac{1}{N} \sum_{i=1}^N \delta_{(k_i,v_i)},
        \qquad
        a_i(q) \coloneqq a_{k_i}(q\mid P_N),
        \qquad \text{and} \qquad
        y(q)\coloneqq \att{q}{P_N}.
    \]
    By the definition,
    \[
        \trace(\cov_{P_N,\delta_q})
        =
        \frac{1}{N} \sum_{i=1}^N
        \left[
            V^2(a_i(q)-1)^2
            +
            a_i(q)^2\|v_i-y(q)\|_2^2
        \right].
    \]

    At \(q=0\), we have \(a_i(0)=1\) for all \(i\), and \(\att{0}{P_N}=0\) by \Cref{lem:lookup-output-geometry}. Hence
    \[
        \trace(\cov_{P_N,\delta_0})
        =
        \frac{1}{N} \sum_{i=1}^N \|v_i\|_2^2
        =
        \rho_0^2V^2
        \asymp
        V^2.
    \]

    Now, fix \(k\in\supp(\nu_N)\).
    By \Cref{lem:lookup-instance-membership}, the keys are distinct, so \(k=k_i\) for a unique \(i\in[N]\).
    Since \(\sum_{\ell=1}^N a_\ell(k_i) = N\), the sharpness condition gives
    \[
        \sum_{\ell\neq i} a_\ell(k_i)
        \le
        N^{-1/2},
        \qquad \text{and hence} \qquad
        a_i(k_i)
        \ge
        N-N^{-1/2}.
    \]
    Therefore, the scalar part alone gives
    \[
        \trace(\cov_{P_N,\delta_{k_i}})
        \ge
        \frac{V^2}{N}(a_i(k_i)-1)^2
        \gtrsim
        V^2N.
    \]
    Conversely,
    \[
        \sum_{\ell=1}^N a_\ell(k_i)^2
        =
        a_i(k_i)^2 + \sum_{\ell\neq i} a_\ell(k_i)^2
        \le
        N^2 + \left(\sum_{\ell\neq i} a_\ell(k_i)\right)^2
        \le
        N^2+1.
    \]
    Also, by \Cref{lem:lookup-output-geometry} and the boundary condition of \Cref{def:lookup-instance},
    \[
        \|y(k_i)\|_2\le V
        \qquad \text{and} \qquad
        \|v_\ell\|_2\le V
    \]
    for all \(\ell\in[N]\), and hence \(\|v_\ell-y(k_i)\|_2\le 2V\). Thus
    \begin{align*}
        \trace(\cov_{P_N,\delta_{k_i}})
        &\le
        \frac{V^2}{N} \sum_{\ell=1}^N (a_\ell(k_i)-1)^2
        +
        \frac{4V^2}{N} \sum_{\ell=1}^N a_\ell(k_i)^2
        \\&\lesssim
        \frac{V^2}{N}
        \left(
            N + \sum_{\ell=1}^N a_\ell(k_i)^2
        \right)
        \lesssim
        V^2N.
    \end{align*}
    Since \(N=2K\), this proves
    \[
        \trace(\cov_{P_N,\delta_k})
        \asymp
        V^2K
    \]
    for all \(k\in\supp(\nu_N)\).

    Finally, the trace is affine in the query distribution, so
    \begin{align*}
        \trace(\cov_{P_N,p\delta_k+(1-p)\delta_0})
        &=
        p \, \trace(\cov_{P_N,\delta_k})
        + (1-p) \, \trace(\cov_{P_N,\delta_0})
        \\&\asymp
        pV^2K + (1-p)V^2
        \asymp
        pV^2K+V^2,
    \end{align*}
    for all \(p\in[0,1]\).
    That proves the claim.
\end{proof}

\begin{proof}[Proof of \Cref{thm:lower-bound-agnostic}]
    Let \(N\coloneqq 2K\), and let \((P_N,\nu_N)\) be the lookup instance from \Cref{lem:lookup-instance-membership}.
    Set
    \[
        p \coloneqq \frac{T}{V^2K}.
    \]
    Since \(T\in[V^2,V^2K]\), we have \(p\in[1/K,1]\).
    For every \(k\in\supp(\nu_N)\), set
    \[
        \nu_k
        \coloneqq
        p \, \delta_k + (1-p) \, \delta_0,
    \]
    and define
    \[
        \mathcal F
        \coloneqq
        \{(P_N,\nu_k) \mid k\in\supp(\nu_N)\}.
    \]
    By \Cref{lem:lookup-instance-membership}, \(\mathcal F \subseteq \kvmeasures*[2K] \times \quemeasures\).
    By \Cref{lem:lookup-pointwise-covariance}, for every \(k\in\supp(\nu_N)\),
    \[
        \trace(\cov_{P_N,\nu_k})
        \asymp
        pV^2K+V^2
        =
        T+V^2
        \asymp
        T,
    \]
    where the last comparison uses \(T\ge V^2\).
    Hence, every \((P,\nu)\in\mathcal F\) satisfies
    \[
        \trace(\cov_{P,\nu}) \asymp T.
    \]

    We next prove the error lower bound.
    Fix any query-agnostic \(K\)-atomic compressor \(A\), and let \(\hat P\) be its random output on the common context \(P_N\).
    For every realization of \(\hat P\),
    \begin{align*}
        &\phantom{=}\frac{1}{N} \sum_{k\in\supp(\nu_N)}
        \err*{P_N}{\nu_k}{\hat P}
        \\&=
        \frac1N\sum_{k\in\supp(\nu_N)}
        \left[
            p\|\att{k}{P_N}-\att{k}{\hat P}\|_2^2
            +(1-p)\|\att{0}{P_N}-\att{0}{\hat P}\|_2^2
        \right] \\
        &\ge
        p\,\err*{P_N}{\nu_N}{\hat P}
        \gtrsim
        pV^2
        =
        \frac{T}{K},
    \end{align*}
    where the last inequality uses \Cref{lem:hard-query-lower-bound}.
    Taking expectation over the randomness of \(A\) gives
    \[
        \frac{1}{N} \sum_{k\in\supp(\nu_N)}
        \Ex\!\left[\err*{P_N}{\nu_k}{\hat P}\right]
        \gtrsim
        \frac{T}{K}.
    \]
    Hence,
    \[
        \sup_{(P,\nu)\in\mathcal F}
        \Ex\!\left[\err*{P}{\nu}{A(P)}\right]
        \gtrsim
        \frac{T}{K}.
    \]
    Taking the infimum over all query-agnostic \(K\)-atomic compressors gives the claim.

    Finally, we prove that the same family is exactly approximable by query-aware \(K\)-atomic compressors.
    By \Cref{lem:two-query-fitting-lookup}, for every \(k\in\supp(\nu_N)\), there exists \(\hat P_k\in\kvmeasures[2]\subseteq\kvmeasures[K]\) such that
    \[
        \err*{P_N}{\nu_k}{\hat P_k}=0.
    \]
    Since the query-aware compressor may depend on \(\nu_k\), this gives
    \[
        \minimax{\compspace{K}}{\mathcal F}=0.
        \tag*{\qedhere}
    \]
\end{proof}

\newpage
\section{Algorithms} \label{apx:algorithms}

This appendix gives proofs and instantiations for the algorithmic framework of \Cref{sec:algorithms}.
\Cref{alg:prefill-compression} and \Cref{alg:decoding-compression} are parameterized by a local compression subroutine \(\reduce\).
Any choice for \(\reduce\) compressing a context from \(2K\) to \(K\) points yields an efficient and parallel KV compression algorithm for causally masked prefill via \Cref{alg:prefill-compression} and autoregressive decoding via \Cref{alg:decoding-compression}.
Recall that we use the following design criteria for \(\reduce\), under which we obtain guarantees on the compression quality.

\defdesigncriteria*

When \(\reduce\) satisfies these criteria, the global merge-reduce scheme matches the minimax-optimal query-agnostic trace risk, up to logarithmic degradation in time, with additional improvements when the supplied feature map \(\Phi\) captures relevant aspects of the KV cache to preserve.

\paragraph{Roadmap.}
This section is organized as follows.
In \Cref{apx:feature-map}, we propose one possible choice for the KV feature map\(\Phi\).
\Cref{apx:reduce-instantiations} shows that \Cref{def:admissible-reducer} can guide concrete algorithm design by constructing two local reducers satisfying the required criteria.
\Cref{apx:global-risk-guarantee} proves \Cref{thm:alg-risk}, which lifts these local guarantees to the full prefill and decoding schemes.
Finally, \Cref{apx:runtime} proves \Cref{pro:runtime} on the efficiency of the global merge-reduce scheme.

\paragraph{Notation.}
We briefly recall the notation used throughout this section.
We work with a stream \((k_i,v_i)_{i\ge1}\) of key--value pairs.
For an index set \(I\subseteq\N\), we write \(P_I\) for the context measure associated with \((k_i,v_i)_{i\in I}\), and abbreviate
\(P_t\coloneqq P_{\{1,\ldots,t\}}\).
We also use the unnormalized mass
\[
    w(I)
    \coloneqq
    \sum_{i\in I}\exp(\|k_i\|^2/(2\sqrt{d_k})),
\]
which comes from the Gaussian-kernel form of softmax attention in \Cref{eq:attention}.
Finally, we fix a bounded feature map \(\Phi:\kvspace\to\hilbert\) for this appendix and use the notation \(\mu_P^\Phi\) and \(\cov_P^\Phi\) from \Cref{def:admissible-reducer} for context measures \(P \in \kvmeasures\). 

\subsection[Choice of feature map]{Choice of feature map \(\Phi\)} \label{apx:feature-map}

The choice of feature map \(\Phi: \kvspace \to \hilbert\) determines which structure a query-agnostic compressor can exploit beyond random sampling.
While there is no single choice that results in a good proxy for all query distributions, a natural query-agnostic choice is obtained from the two linear quantities that define attention before normalization: the unnormalized, value-weighted numerator and the scalar normalizer.

Concretely, inspired by the response profile \(\Gamma_P\), we choose
\begin{equation}
    \label{eq:reference-features}
    \Phi:\kvspace\to\hilbert,
    \qquad
    (k,v)\mapsto
    \bigl(\kappa(\cdot,k)\otimes v,\;V\,\kappa(\cdot,k)\bigr),
\end{equation}
where
\[
    \hilbert
    \coloneqq
    \hilbert_\kappa\otimes\R^{d_v}
    \oplus
    \hilbert_\kappa
\]
is constructed from the feature space \(\hilbert_\kappa\) of the Gaussian kernel \(\kappa\) from \Cref{sec:setting}.
This choice of feature map assigns equal importance to all KV directions and is therefore a reasonable choice for when the query distribution does not address some keys much more than others.
To make the feature map finite-dimensional, we project them onto the top-\(d\) eigenspace of their covariance under the present context measure in our experiments.
Note, however, that our results do not depend on this particular choice but work for arbitrary feature maps \(\Phi\).
\subsection{Admissible reducers} \label{apx:reduce-instantiations}

We instantiate the abstract design criteria of \Cref{def:admissible-reducer} by giving two concrete local reducers. 
The first reducer is a constrained rounding procedure derived directly from the sparse barycenter approximation theorem in \Cref{apx:sparse-hilbert-approximation}. 
It provides the cleanest theoretical guarantee because it is \((\Phi,0)\)-admissible, meaning it protects the chosen features \(\Phi\) exactly.
However, the constrained rounding routine is inherently sequential in the protected dimension \(d\), making it poorly suited to modern parallel hardware.
To compensate for this, we also provide a more practical, clustering-based reducer.
It maps efficiently to parallel hardware, but it is only \((\Phi,\tau)\)-admissible with a non-zero \(\tau\) term determined by the achieved clustering error.

\begin{algorithm}[t]
\caption{\(\textsc{RoundReduce}\)}
\label{alg:constrained-reduce}
\begin{algorithmic}[1]
    \Require context measure \(Q=\sum_{i=1}^{s} q_i \, \delta_{z_i}\in\kvmeasures[2K]\)
    \Require bounded feature map \(\Phi: \kvspace \to \R^d\), \(d\le (K-3)/3\)

    \State \(\mu \gets \sum_{i=1}^s q_i \, \delta_{\Phi(z_i)}\) \Comment{pushforward}
    \State \(\hat \mu \gets \textsc{ProtectedSparsify}(\mu, \R^d, K-2d-2) = \sum_{i=1}^s \hat q_i \, \delta_{\Phi(z_i)}\) \Comment{\Cref{alg:protected-hilbert-round}}
    \State \Return \(\hat Q \gets \sum_{i=1}^s \hat q_i \, \delta_{z_i}\) \Comment{pullback}
\end{algorithmic}
\end{algorithm}

\begin{algorithm}[t]
\caption{\(\textsc{ClusterReduce}\)}
\label{alg:cluster-reduce}
\begin{algorithmic}[1]
    \Require context measure \(Q=\sum_{i=1}^{2K} q_i \, \delta_{z_i}\in\kvmeasures[2K]\)
    \Require bounded feature map \(\Phi: \kvspace \to \R^d\)

    \State \(H\gets\{i \mid q_i>1/K\}\), \quad \(L\gets [2K] \setminus H\) \Comment{filter out heavy-weight atoms}
    \If{\(\sum_{i\in L}q_i=0\)}
        \State \Return \(Q\)
    \EndIf
    \State \(m\gets K-|H|\), \quad \(\beta\gets\sum_{i\in L}q_i\), \quad \(\rho\gets \beta/m\)

    \vspace{0.5em}
    \For{\(i\in L\)}
        \State \(u_i\gets\Phi(z_i)\)
        \State \(p_i\gets q_i/\rho\)
    \EndFor
    \State \(X\gets\textsc{ClusterSlots}((p_i,u_i)_{i\in L},m) \in [0,1]^{m \times |L|}\)
    \vspace{0.5em}
    
    \For{\(a=1,\ldots,m\)} \Comment{sample one atom per cluster}
        \State independently sample \(I_a\in L\) with \(\mathbb P(I_a=i)=X_{ai}\)
    \EndFor
    
    \vspace{0.5em}
    \State \Return \(
        \hat Q
        \gets
        \sum_{i\in H}q_i\delta_{z_i}
        +
        \sum_{a=1}^m \rho\,\delta_{z_{I_a}} .
    \)
\end{algorithmic}
\end{algorithm}

\subsubsection{Perfect admissibility via constrained rounding}

\Cref{alg:constrained-reduce} is a direct implementation of the Hilbert-space sparsification argument in \Cref{apx:sparse-hilbert-approximation}. 
Given a context measure \(Q \in \kvmeasures[2K]\), the reducer preserves the features \(\Phi\) exactly and sparsifies the measure only in the orthogonal complement.
Because of that, the reducer is \((\Phi,0)\)-admissible.

\begin{theorem} \label{thm:constrained-reduce-admissible}
    Let \(K\ge 3\), \(0\le d\le (K-3)/3\) be an integer, and let \(\Phi:\kvspace\to\R^d\) be bounded.
    Then \(\textsc{RoundReduce}(\cdot,\Phi)\) is \((\Phi,0)\)-admissible.
\end{theorem}
\begin{proof}
    Fix
    \[
        Q = \sum_{i=1}^s q_i \, \delta_{z_i} \in \kvmeasures[2K],
    \]
    set \(m \coloneqq K-2d-2\), and use the notation of \Cref{alg:constrained-reduce}.
    Since \(d\le (K-3)/3\), we have \(m \ge K/3\).
    
    \(\textsc{RoundReduce}(Q,\Phi)\) applies \(\textsc{ProtectedSparsify}\) to the pushforward
    \[
        \mu = \sum_{i=1}^s q_i\delta_{\Phi(z_i)}
    \]
    with protected subspace \(U = \R^d\) and sparsification budget \(m\).
    
    By \Cref{pro:protected-sparsification},
    \[
        |\supp(\hat Q)| \le m+2\dim(U)+2 \le K
    \]
    almost surely, and \(\supp(\hat Q)\subseteq\supp(Q)\).
    Hence,
    \[
        \hat Q\in\Prob_K(\supp(Q)).
    \]
    The rounding keeps the original atom labels even when their projected coordinates coincide. In the proof of \Cref{pro:protected-sparsification}, each labeled weight is unbiased, so \(\Ex[\hat Q]=Q\).
    That proves unbiasedness.
    
    Let \(\mathcal G\) be a Hilbert space and let \(\Psi:\kvspace\to\mathcal G\) be bounded.
    Applying the variance bound in \Cref{pro:protected-sparsification} with \(y_i=\Psi(z_i)\) gives
    \[
        \Ex \Bigl\| \int_{\kvspace}\Psi\,d(\hat Q-Q) \Bigr\|_{\mathcal G}^2
        \le
        \frac{20}{9m}\trace(\cov_Q^\Psi)
        \le
        \frac{20}{3K}\trace(\cov_Q^\Psi).
    \]
    
    Since \(U=\R^d\), exact preservation on \(U\) in \Cref{pro:protected-sparsification} gives \(\int\Phi\,\diff(\hat Q-Q)=0\) almost surely.

    Thus, all design criteria hold for every \(Q \in \mathcal P_{2K}(X)\), making \(\textsc{RoundReduce}(\cdot,\Phi)\) \((\Phi,0)\)-admissible.
\end{proof}

\subsubsection{Practical admissibility via clustering}

\Cref{alg:cluster-reduce} provides an admissible reducer based on clustering.
In particular, for the clustering step, any finite-dimensional clustering subroutine \(\textsc{ClusterSlots}\) can be used, provided it produces a balanced latent-slot assignment.
Concretely, using the notation of \Cref{alg:cluster-reduce}, the subroutine receives the weighted protected coordinates \((p_i,u_i)_{i\in L}\), where \(p_i=q_i/\rho\), and returns a matrix \(X \in [0,1]^{m\times |L|}\).
Row \(a\) of \(X\) is the sampling distribution used by latent slot \(a\), while the column constraints ensure that each light atom receives the correct expected mass.
Formally, \(X\) must satisfy
\[
    \sum_{i\in L}X_{ai}=1
    \qquad \text{and} \qquad
    \sum_{a=1}^m X_{ai}=p_i.
\]
For example, one might use weighted k-means to group nearby protected coordinates and then sample representatives from those groups in proportion to their total mass, with \(X\) recording the resulting sampling probabilities.

Naturally, the quality of the compressed summary depends on the quality of the clustering.
Using the notation of \Cref{alg:cluster-reduce}, define the clustering error of the call to \textsc{ClusterSlots} by
\[
    \mathcal C(Q)
    \coloneqq
    \frac{1}{K}
    \sum_{a=1}^m\sum_{i\in L}
    X_{ai}\|u_i-\bar u_a\|_2^2,
    \qquad
    \bar u_a \coloneqq \sum_{i\in L}X_{ai}u_i.
\]
If the light mass is zero, we set \(\mathcal C_{\textsc{ClusterSlots}}(Q)=0\).

When this clustering error is small, \textsc{ClusterReduce} is admissible, as the following theorem shows.

\begin{theorem} \label{thm:protect-cluster-reduce-admissible}
    Fix \(K \ge 1\) and \(r \ge 0\).
    Let \(\hilbert\) be a Hilbert space and \(\Phi: \kvspace \to \R^d\) bounded.
    Assume that \(\textsc{ClusterSlots}\) satisfies \(\mathcal C(Q)\le \tau\) for every input \(Q\in\kvmeasures[2K]\).
    Then, \(\textsc{ClusterReduce}(\cdot,\Phi)\) is \((\Phi,\tau)\)-admissible.
\end{theorem}
\begin{proof}
    Let
    \[
        Q=\sum_{i=1}^{2K}q_i\delta_{z_i}\in\kvmeasures[2K].
    \]
    
    Use the notation of \Cref{alg:cluster-reduce}.
    If the light mass is zero, then \(Q\) must already belong to \(\kvmeasures[K]\).
    The algorithm then returns \(Q\) exactly, so all claims are immediate.
    
    Assume therefore that \(\beta>0\).
    Since each \(i\in H\) satisfies \(q_i>1/K\),
    \[
        \beta
        =
        1-\sum_{i\in H}q_i
        \le
        1-\frac{|H|}{K}
        =
        \frac{K-|H|}{K}
        =
        \frac{m}{K}.
    \]
    Thus, \(\rho=\frac{\beta}{m}\le \frac{1}{K}\). Condition on the slot assignment \(X\). The support, unbiasedness, and universal variance bounds below hold for every balanced assignment.
    
    The output is supported on \(\supp(Q)\) and uses at most \(|H|+m=K\) atoms, so \(\hat Q\in\Prob_K(\supp(Q))\) almost surely.
    Moreover, the heavy atoms are kept exactly, and for every \(i\in L\),
    \[
        \Ex[\text{mass assigned to }z_i]
        =
        \rho\sum_{a=1}^m X_{ai}
        =
        \rho p_i
        =
        q_i.
    \]
    Hence \(\Ex[\hat Q]=Q\). That proves unbiasedness.
    
    Let \(\mathcal G\) be a Hilbert space and let \(\Psi:\kvspace\to\mathcal G\) be bounded. Write
    \[
        y_i \coloneqq \Psi(z_i),
        \qquad
        \bar y \coloneqq \int_{\kvspace}\Psi \diff Q.
    \]
    The heavy atoms contribute no error, while the light slots give
    \[
        \int_{\kvspace}\Psi \diff(\hat Q-Q)
        =
        \rho\sum_{a=1}^m
        \left( y_{I_a}-\sum_{i\in L}X_{ai}y_i \right).
    \]
    The summands are independent and centered. Therefore
    \[
    \begin{aligned}
        \Ex\left[
            \left\|
                \int_{\kvspace}\Psi \diff(\hat Q-Q)
            \right\|_{\mathcal G}^2
        \right]
        &=
        \rho^2
        \sum_{a=1}^m
        \Ex\left\|
            y_{I_a}-\sum_{i\in L}X_{ai}y_i
        \right\|_{\mathcal G}^2  \\
        &\le
        \rho^2
        \sum_{a=1}^m\sum_{i\in L}
        X_{ai}\|y_i-\bar y\|_{\mathcal G}^2  \\
        &=
        \rho^2
        \sum_{i\in L}
        p_i\|y_i-\bar y\|_{\mathcal G}^2  \\
        &=
        \rho
        \sum_{i\in L}
        q_i\|y_i-\bar y\|_{\mathcal G}^2  \\
        &\le
        \frac1K
        \sum_{i=1}^{2K}
        q_i\|y_i-\bar y\|_{\mathcal G}^2  \\
        &=
        \frac1K\trace(\cov_Q^\Psi).
    \end{aligned}
    \]
    This proves the universal trace criterion.
    
    For criterion (iii), use \(\Psi=\Phi\), so the slot means in the defect are \(\bar u_a\). Conditional independence and centering give
    \begin{align*}
        \Ex\left[\left\|\int_{\kvspace}\Phi\,\diff(\hat Q-Q)\right\|_2^2\middle|X\right]
        &= \rho^2\sum_{a=1}^m\sum_{i\in L}X_{ai}\|u_i-\bar u_a\|_2^2 \\
        &= K\rho^2\mathcal C(Q)
        \le \frac{\mathcal C(Q)}{K}.
    \end{align*}
    Taking expectations over the assignment gives the bound \(\tau/K\).
\end{proof}

\subsection{Global risk guarantees} \label{apx:global-risk-guarantee}
This section proves the risk guarantee of \Cref{thm:alg-risk} for the chunked prefill and decoding algorithms.
Both schedules build summaries in a merge-reduce tree: starting from token chunks of \(K\) tokens, disjoint summaries are merged, their weighted union is compressed back to \(K\) atoms, and the unfinished current chunk is appended exactly.
We show that, for admissible local compressors \(\reduce\), the approximation error accumulates linearly with the depth of the merge-reduce tree, which then transfers to the claimed risk guarantees.

\subsubsection{Risk propagation through merge-reduce trees}
We now turn the one-step guarantees of \(\reduce\) into guarantees for full causal prefixes.
To that end, we show that the design criteria in \Cref{def:admissible-reducer} for \(\reduce\), unbiasedness of the summary and second-moment control of its feature defect, propagate through the merge-reduce recursion.
To express this invariant, we introduce the notion of a \emph{valid summary}, which is a random context measure of a token block that satisfies the same criteria, with a parameter counting the number of reducer layers on its dependency path.
Each node of the merge-reduce tree preserves validity, increasing this parameter by at most one.
The global risk bound in \Cref{thm:alg-risk} then follows by recognizing the prefill and decoding summaries as outputs of such trees, whose depth is logarithmic in the number of completed chunks.

\begin{definition}[Valid summary] \label{def:validity}
    Let \(d,\tau \ge 0\) and \(A \subseteq \N\).
    We say that a random context measure \(S_A\) is \emph{\((d,\tau)\)-valid} for \(A\) if the following hold:
    \begin{enumerate}[label=(\roman*)]
        \item \(\Ex[S_A] = P_A\)
        \item \(
            \Ex\bigl[\|\int_{\kvspace}\Psi \diff(S_A-P_A)\|_{\mathcal G}^2\bigr]
            \lesssim
            \frac{d}{K} \trace(\cov_{P_A}^\Psi)
        \) for all Hilbert spaces \(\mathcal G\) and bounded \(\Psi: \kvspace\to\mathcal G\)
        \item \(
            \Ex\bigl[\|\int_{\kvspace}\Phi\,\diff(S_A-P_A)\|_2^2\bigr]
            \le
            \frac{d\tau}{K}
        \)
    \end{enumerate}
\end{definition}

Note that this resembles directly the design criteria of a \((\Phi,\tau)\)-admissible reducer, as defined in \Cref{def:admissible-reducer}.
In particular, if \(\reduce\) is \((\Phi,\tau)\)-admissible and \(P_A \in \kvmeasures[2K]\) is the context measure of the tokens in \(A \subseteq \N\), then \(\reduce(P_A)\) is \((1,\tau)\)-valid for \(A\).

Throughout this section, \(d\) denotes depth and \(r\) is the dimension of the fixed feature map \(\Phi: \kvspace\to\R^r\).
We briefly record the following fact that we use repeatedly later on.

\begin{lemma} \label{lem:covariance-domination}
    Let \(P \in \kvmeasures\) and let \(Q\) be a random variable taking values in \(\kvmeasures\).
    If \(\Ex[Q] = P\), then \(\Ex[\cov^\Psi_Q] \preceq \cov^\Psi_P\) for any Hilbert space \(\mathcal G\) and bounded \(\Psi: \kvspace \to \mathcal G\).
\end{lemma}
\begin{proof}
    Let \(\mathcal G\) be a Hilbert space and \(\Psi: \kvspace \to \mathcal G\) be bounded.
    Then,
    \begin{align*}
        \Ex\!\left[\cov_{Q}^\Psi\right]
        &=
        \Ex\!\left[ \int_\kvspace \Psi \otimes \Psi \diff Q \right] - \Ex\!\left[ \mu_Q^\Psi \otimes \mu_Q^\Psi \right]
        \\&=
        \int_\kvspace \Psi \otimes \Psi \diff P - \Ex\!\left[ \mu_Q^\Psi \otimes \mu_Q^\Psi \right]
        \\&=
        \int_\kvspace \Psi \otimes \Psi \diff P - \Ex\!\left[ \left(\mu_Q^\Psi - \Ex\!\left[\mu_Q^\Psi\right]\right) \otimes \left(\mu_Q^\Psi - \Ex\!\left[\mu_Q^\Psi\right]\right) \right] - \Ex\!\left[\mu_Q^\Psi\right] \otimes \Ex\!\left[\mu_Q^\Psi\right]
        \\&=
        \int_\kvspace \Psi \otimes \Psi \diff P - \Ex\!\left[ \left(\mu_Q^\Psi - \mu_P^\Psi\right) \otimes \left(\mu_Q^\Psi - \mu_P^\Psi\right) \right] - \mu_P^\Psi \otimes \mu_P^\Psi
        \\&=
        \cov_P^\Psi - \Ex\!\left[ \left(\mu_Q^\Psi - \mu_P^\Psi\right) \otimes \left(\mu_Q^\Psi - \mu_P^\Psi\right) \right]
        \\&\preceq
        \cov_P^\Psi.
        \tag*{\qedhere}
    \end{align*}
\end{proof}

We now consider the operation at a single node in the merge-reduce tree.
Given two summaries of two disjoint blocks of tokens, we form a compressed summary of their union by compressing their weighted combination using \(\reduce\).
When the two input summaries are \((d,\tau)\)-valid, the resulting summary is \((d+1,\tau)\)-valid, thereby propagating the guarantees up the merge-reduce tree.

\begin{lemma} \label{lem:one-merge-validity}
    Let \(A,B \subseteq \N\) be finite, nonempty, and disjoint.
    Let \(S_A\) be \((d_A,\tau)\)-valid for \(A\), and let \(S_B\) be \((d_B,\tau)\)-valid for \(B\).
    Assume that \(\reduce\) is \((\Phi,\tau)\)-admissible and that \(S_A\) and \(S_B\) are independent.
    Define
    \[
        Q \coloneqq \frac{w(A)}{w(A)+w(B)} S_A + \frac{w(B)}{w(A)+w(B)} S_B.
    \]
    If \(Q\in\kvmeasures[2K]\), then \(\reduce(Q)\) is \((1+\max\{d_A,d_B\},\tau)\)-valid for \(A\cup B\).
\end{lemma}
\begin{proof}
    Let \(S_{A \cup B} \coloneqq \reduce(Q)\), and write
    \[
        P\coloneqq P_{A\cup B}
        =
        \alpha P_A+\beta P_B
    \]
    for \(\alpha \coloneqq w(A) / (w(A) + w(B))\) and \(\beta \coloneqq w(B) / (w(A) + w(B))\).
    Additionally, introduce the defects
    \[
        \sigma_A \coloneqq S_A - P_A
        ,\qquad
        \sigma_B \coloneqq S_B - P_B
        ,\qquad
        \sigma_{A \cup B} \coloneqq S_{A \cup B} - P_{A \cup B},
    \]
    and
    \[
        \eta \coloneqq S_{A \cup B} - Q
    \]
    which are all zero-mass, signed measures on \(\kvspace\).
    
    For a Hilbert space \(\mathcal G\) and bounded \(\Psi: \kvspace \to \mathcal G\), introduce the notation
    \[
        \Psi[\sigma] \coloneqq \int_\kvspace \Psi \diff \sigma
    \]
    for zero-mass, signed measures \(\sigma\) on \(\kvspace\).

    We next prove orthogonality of the defects after taking expectation.
    Condition on the already constructed summaries \(S_A,S_B\).
    Then \(Q\), \(\sigma_A\), and \(\sigma_B\) are fixed, and the only remaining randomness is the fresh call to \(\reduce\).
    Admissibility of \(\reduce\) gives \(\Ex[\eta\mid S_A,S_B]=0\).
    Hence,
    \[
        \Ex\bigl[\langle \Psi[\eta], \Psi[\sigma_A] \rangle_\mathcal G\bigr]
        =
        \Ex\bigl[\langle \Ex[\Psi[\eta] \mid S_A, S_B], \Psi[\sigma_A] \rangle_\mathcal G\bigr]
        =
        0,
    \]
    and similarly for \(\Ex[\langle \Psi[\eta], \Psi[\sigma_B] \rangle_\mathcal G]\).
    Moreover, \(S_A\) and \(S_B\) are independent and validity of \(S_A\) and \(S_B\) imply
    \[
        \Ex[\Psi[\sigma_A]] = \Ex[\Psi[\sigma_B]] = 0.
    \]
    Therefore,
    \[
        \Ex\bigl[\langle \Psi[\sigma_A], \Psi[\sigma_B] \rangle_\mathcal G\bigr]
        =
        \bigl\langle \Ex\bigl[\Psi[\sigma_A]\bigr], \Ex\bigl[\Psi[\sigma_B]\bigr] \bigr\rangle_\mathcal G
        =
        0
    \]
    and thus expanding the square norm gives
    \begin{equation} \label{eq:one-merge-validity:pythagoras}
        \Ex\bigl\|\Psi[\sigma_{A \cup B}]\bigr\|_\mathcal G^2
        =
        \Ex\bigl\|\Psi[\eta]\bigr\|_\mathcal G^2
        +
        \alpha^2 \, \Ex\bigl\|\Psi[\sigma_A]\bigr\|_\mathcal G^2
        +
        \beta^2 \, \Ex\bigl\|\Psi[\sigma_B]\bigr\|_\mathcal G^2.
    \end{equation}

    We now prove the three validity conditions for \(S_{A \cup B}\).

    For unbiasedness,
    \[
        \Ex[S_{A\cup B}]
        =
        \Ex[\Ex[\reduce(Q)\mid S_A,S_B]]
        =
        \Ex[Q]
        =
        \alpha P_A+\beta P_B
        =
        P,
    \]
    which follows immediately from unbiasedness of \(S_A\) and \(S_B\).

    For trace control, we bound the three terms in \Cref{eq:one-merge-validity:pythagoras}.
    Fix a Hilbert space \(\mathcal G\) and a bounded map \(\Psi: \kvspace \to \mathcal G\).
    Since \(\Ex[Q]=P\), \Cref{lem:covariance-domination} gives \(\Ex[\cov^\Psi_Q]\preceq \cov^\Psi_P\).
    Conditioning on \(S_A,S_B\), the input \(Q\) is fixed and the remaining randomness is the fresh reducer randomness.
    Thus, admissibility of \(\reduce\) gives
    \[
        \Ex\bigl[\bigl\| \Psi[\eta] \bigr\|_\mathcal G^2 \bigm| S_A,S_B \bigr]
        \lesssim
        \frac{1}{K}\trace\bigl(\cov_Q^\Psi\bigr).
    \]
    Taking expectations,
    \[
        \Ex\bigl\| \Psi[\eta] \bigr\|_\mathcal G^2
        \lesssim
        \frac{1}{K}\Ex\trace\bigl(\cov_Q^\Psi\bigr)
        \le
        \frac{1}{K}\trace\bigl(\cov_P^\Psi\bigr).
    \]

    By validity of \(S_A\) and \(S_B\),
    \begin{align*}
        \alpha^2\Ex\bigl\|\Psi[\sigma_A]\bigr\|_{\mathcal G}^2
        +
        \beta^2\Ex\bigl\|\Psi[\sigma_B]\bigr\|_{\mathcal G}^2
        &\lesssim
        \frac{\max\{d_A,d_B\}}{K}
        \Bigl[
            \alpha^2\trace(\cov^\Psi_{P_A})
            +
            \beta^2\trace(\cov^\Psi_{P_B})
        \Bigr]
        \\&\le
        \frac{\max\{d_A,d_B\}}{K}
        \Bigl[\trace\bigl(\alpha \cov^\Psi_{P_A} + \beta \cov^\Psi_{P_B}\bigr)\Bigr]
        \\&\le
        \frac{\max\{d_A,d_B\}}{K} \Bigl[\trace\bigl(\cov^\Psi_P\bigr)\Bigr],
    \end{align*}
    where we additionally use \(\alpha^2\le\alpha\), \(\beta^2\le\beta\), and \(\alpha \cov^\Psi_{P_A}+\beta \cov^\Psi_{P_B}\preceq \cov^\Psi_P\) by \Cref{lem:covariance-domination}.
    Combining this with \Cref{eq:one-merge-validity:pythagoras} and admissibility of \(\reduce\) proves
    \[
        \Ex\bigl\| \Psi[\sigma_{A \cup B}] \bigr\|_\mathcal G^2
        \lesssim
        \left(\frac{1+\max\{d_A,d_B\}}{K}\right) \trace\bigl(\cov^\Psi_P\bigr).
    \]
    Since \(\Psi\) and \(\mathcal G\) were arbitrary, that proves the trace condition.

    For the feature-mean bound, apply \Cref{eq:one-merge-validity:pythagoras} to \(\Psi=\Phi\).
    Criterion (iii) of \Cref{def:admissible-reducer}, applied conditionally on the child summaries, bounds the local contribution by \(\tau/K\). By validity of the children,
    \[
        \Ex\|\Phi[\sigma_{A\cup B}]\|_2^2
        \le \frac{\tau}{K}(1+\alpha^2d_A+\beta^2d_B)
        \le \frac{\tau}{K}(1+\max\{d_A,d_B\}).
    \]
    The last inequality uses \(\alpha^2+\beta^2\le 1\). Thus \(S_{A\cup B}\) is \((1+\max\{d_A,d_B\},\tau)\)-valid.
\end{proof}

We now transfer the validity bounds to attention risk.

\begin{proposition} \label{pro:algorithm-risk}
    Fix \(\nu\in\quemeasures\). If \(\hat P_t\) is \((d_t,\tau)\)-valid for \(\{1,\ldots,t\}\), then
    \[
        \Ex\err*{P_t}{\nu}{\hat P_t}
        \lesssim \frac{d_t}{K}
        \inf_L \left\{\trace\bigl(\cov_{P_t}^{\Gamma_{P_t}-L\circ\Phi}\bigr)
        +\|L\|_{\mathrm{op}}^2\tau\right\},
    \]
    where the infimum is over bounded linear maps \(L:\R^r\to\hilbert_\nu\).
\end{proposition}
\begin{proof}
    Since \(\hat P_t\) is unbiased and is a probability measure, its mass outside \(\supp(P_t)\) is zero almost surely. Thus
    \begin{equation} \label{eq:algorithm-risk:transfer}
        \err*{P_t}{\nu}{\hat P_t}
        \overset{(1)}{\lesssim} \left\|
            \int_{\kvspace}\Gamma_{P_t} \diff\hat P_t
        \right\|_{\hilbert_\nu}^2
        \overset{(2)}{=}
        \left\|
            \int_{\kvspace}\Gamma_{P_t} \diff(\hat P_t - P_t)
        \right\|_{\hilbert_\nu}^2
    \end{equation}
    where \((1)\) uses \Cref{lem:error-hilbert-transfer}, and \((2)\) uses \(\int_\kvspace \Gamma_{P_t} \diff P_t = 0\).
    
    Fix a bounded linear map \(L:\R^r\to\hilbert_\nu\) and write \(\Delta_t=\hat P_t-P_t\).
    Decompose the response defect as
    \[
        \int_{\kvspace}\Gamma_{P_t}\,\diff\Delta_t
        = \int_{\kvspace}(\Gamma_{P_t}-L\circ\Phi)\,\diff\Delta_t
        + L\int_{\kvspace}\Phi\,\diff\Delta_t.
    \]
    The residual is bounded on \(\supp(P_t)\). For each atom \(z=(k,v)\) in this support, \(a_k(q\mid P_t)\le 1/P_t(\{z\})\) for all queries \(q\).
    Extend the residual by zero off \(\supp(P_t)\) to apply the universal variance bound in \Cref{def:validity} to the first term. By the feature-mean bound, the expected squared norm of the second term is at most \(\|L\|_{\mathrm{op}}^2d_t\tau/K\).
    Using \(\|a+b\|^2\le 2(\|a\|^2+\|b\|^2)\) in \eqref{eq:algorithm-risk:transfer} and taking the infimum over \(L\) gives the claim.
\end{proof}

The final summary \(\hat P_t\) used as the compressed context for the token at position \(t\) is a combination of a compressed summary produced by the merge-reduce tree and an exact context measure of the tokens inside the same chunk as token \(t\).
The next lemma shows that this is harmless for the validity of the resulting summary.

\begin{lemma} \label{lem:exact-suffix-validity}
    Let \(A,R\subseteq\N\) be finite, nonempty, and disjoint.
    If \(S_A\) is \((d,\tau)\)-valid for \(A\), then
    \[
        S_{A\cup R}
        \coloneqq
        \frac{w(A)}{w(A)+w(R)}S_A
        +
        \frac{w(R)}{w(A)+w(R)}P_R
    \]
    is \((d,\tau)\)-valid for \(A\cup R\).
\end{lemma}
\begin{proof}
    Unbiasedness of \(S_{A \cup R}\) is immediate from unbiasedness of \(S_A\) and \(P_R\).
    Set \(\alpha \coloneqq w(A)/(w(A)+w(R))\).

    Let \(\mathcal G\) be a Hilbert space and \(\Psi: \kvspace \to \mathcal G\) bounded.
    Then, the defect of \(S_{A\cup R}\) is
    \[
        \int_\kvspace\Psi \diff (S_{A \cup R} - P_{A \cup R})
        =
        \alpha \left( \int_\kvspace\Psi \diff (S_A - P_A) \right)
    \]
    Since \(\alpha^2\le\alpha\) and \(\alpha\cov^\Psi_{P_A}\preceq\cov^\Psi_{P_{A\cup R}}\), this gives the variance bound for \(S_{A\cup R}\).
    Taking \(\Psi=\Phi\) in the same identity and using \(\alpha^2\le 1\) gives the feature-mean bound for \(S_{A\cup R}\).
\end{proof}

It remains only to count the validity depth of the summaries produced by \Cref{alg:prefill-compression,alg:decoding-compression}.
The prefill scan and the decoding buckets both construct summaries by repeated binary merges of disjoint completed chunks.
Along any dependency path there are only logarithmically many reducer calls, and the unfinished current chunk is appended exactly.
Thus, if \(\reduce\) is \((\Phi,\tau)\)-admissible, each resulting prefix summary \(\hat P_t\) is \((d,\tau)\)-valid, with \(d\) proportional to the depth of the tree, which is logarithmic in the number of tokens.
The claimed risk guarantees then follow immediately from \Cref{pro:algorithm-risk}.

\thmalgguarantee*
\begin{proof}
    We first show that \(\hat P_t\) as produced by \Cref{alg:prefill-compression,alg:decoding-compression} is \((\bigO(\log(2+t/K)),\tau)\)-valid for \(\{1,\ldots,t\}\).
    The risk bound then follows immediately from \Cref{pro:algorithm-risk}.

    Let \(I_1,I_2,\ldots\) denote the consecutive chunks of size \(K\), and let \(c(t)\) be the chunk containing \(t\).
    Write
    \[
        A_t \coloneqq I_1\cup\cdots\cup I_{c(t)-1},
        \qquad
        R_t \coloneqq I_{c(t)}\cap\{1,\ldots,t\}.
    \]
    Thus, \(P_t\) is the normalized mixture of \(P_{A_t}\) and the exact current-chunk suffix \(P_{R_t}\), with the obvious convention when \(A_t=\emptyset\).

    We claim that every summary stored by the merge-reduce computation for a union of \(m\) completed chunks is \((\bigO(\log(2+m)),\tau)\)-valid for its underlying index set, and that summaries stored on disjoint index sets are independent.
    This follows by induction on the binary merge tree.
    Each leaf is a single exact chunk, hence is deterministic and \((0,\tau)\)-valid.
    Whenever two disjoint summaries with validity \(d_1\) and \(d_2\) are merged, the induction hypothesis gives their independence; the reducer call uses fresh randomness independent of both summaries.
    Thus, \Cref{lem:one-merge-validity} shows that the resulting summary is \((1+\max\{d_1,d_2\},\tau)\)-valid.
    The resulting summary depends only on its two children and on the fresh randomness used in this reducer call, so it remains independent of all stored summaries supported on disjoint index sets.
    If one input is empty and the reducer is called on the other input alone, the same validity proof applies with a single child and weight one.
    In prefill the prefix is assembled from at most \(1+\lfloor\log_2 m\rfloor\) dyadic blocks of depth at most \(\lfloor\log_2 m\rfloor\), for \(m\ge 1\).
    In decoding both the depth of a bucket and the number of nonempty buckets are \(\bigO(\log(2+m))\).
    Thus along any dependency path in either schedule there are \(\bigO(\log(2+m))\) reductions.
    Hence the summary of any completed-chunk prefix produced by the prefill scan, and the completed-prefix summary maintained by the decoding buckets, is \((\bigO(\log(2+m)),\tau)\)-valid.

    Apply this to the completed chunks preceding \(t\).
    If \(A_t\neq\emptyset\), the algorithm has an \((\bigO(\log(2+|A_t|/K)),\tau)\)-valid summary \(S_{A_t}\) for \(A_t\).
    Since \(R_t\) is kept exact,
    \[
        \hat P_t
        =
        \frac{w(A_t)}{w(A_t)+w(R_t)}S_{A_t}
        +
        \frac{w(R_t)}{w(A_t)+w(R_t)}P_{R_t}.
    \]
    By \Cref{lem:exact-suffix-validity}, and using \(|A_t| \le t\), \(\hat P_t\) is \((\bigO(\log(2+t/K)),\tau)\)-valid for \(\{1,\ldots,t\}\).

    When \(A_t=\emptyset\), the summary is exact and hence \((0,\tau)\)-valid. The claim follows from \Cref{pro:algorithm-risk}.
    The trace bound with \(L=0\) requires only the support requirement and criteria (i) and (ii), since the induction for universal variance and the risk transfer at \(L=0\) use only these conditions.
\end{proof}

The global guarantee therefore reduces entirely to the design of \(\reduce\).
Any concrete reducer satisfying the one-step trace and tail conditions of \Cref{def:admissible-reducer} inherits the same causal prefill and decoding guarantees through the merge-reduce schedules of \Cref{alg:prefill-compression,alg:decoding-compression}, with only the logarithmic depth loss.
\subsection{Computational cost of the global merge-reduce scheme} \label{apx:runtime}
We prove \Cref{pro:runtime} from the main text.

\proruntime*
\begin{proof}
    Let \(M=n/K\) and \(L=\log_2 M\). We count calls to \(\textsc{Reduce}\) and atom storage, treating each \(K\)-atomic summary as an object of size \(K\).
    
    First, consider Algorithm~\ref{alg:prefill-compression}. 
    In the upsweep stage, level \(\ell\) contains \(M/2^\ell\) independent merge-reduce operations.
    Hence, the number of calls to \(\textsc{Reduce}\) in the upsweep is
    \[
        \sum_{\ell=1}^{L} \frac{M}{2^\ell} = M-1 .
    \]
    The downsweep has the same count.
    The total number of calls is thus \(2M-2=\bigO(n/K)\).
    
    The parallel depth is also immediate from the tree structure.
    All calls at a fixed level are independent, so each level contributes one parallel round of \(\textsc{Reduce}\).
    There are \(L\) upsweep levels and \(L\) downsweep levels, giving depth \(\bigO(L)=\bigO(\log(n/K))\).
    
    For memory, the algorithm stores \(\bigO(M)\) summaries across all levels of the scan tree, and each summary is \(K\)-atomic.
    Thus the total number of stored atoms is \(\bigO(MK)=\bigO(n)\).
    
    Now consider Algorithm~\ref{alg:decoding-compression}.
    After \(s\) completed chunks have been incorporated, including the prompt summary, the nonempty buckets satisfy the standard binary-counter invariant: bucket \(B_\ell\), when nonempty, stores a \(K\)-atomic summary of a block of \(2^\ell\) chunks, and there is at most one nonempty bucket at each level.
    Therefore, the number of nonempty buckets is at most \(1+\lfloor \log_2 s \rfloor\).
    For a prefix of length \(t=sK\), the number of stored atoms is consequently
    \[
        \bigO(K\log s)=\bigO(K\log(t/K)).
    \]
    
    It remains to count calls to \(\textsc{Reduce}\).
    Each new chunk first creates a level-zero bucket.
    The while loop then performs the carry operations of a binary counter.
    Whenever a bucket is already occupied, the two \(K\)-atomic summaries are merged and reduced to form one summary at the next level.
    Over \(s\) chunk insertions, level \(\ell\) participates in at most \(O(s/2^{\ell+1})\) such carries.
    Summing over \(\ell\) gives \(\bigO(s)\) carry calls in total, hence \(\bigO(1)\) amortized carry calls per chunk.
    
    After the carry step, the algorithm forms the current history summary by merging the nonempty buckets in decreasing order.
    Since there are \(\bigO(\log s)\) nonempty buckets, this uses \(\bigO(\log s)\) calls to \(\textsc{Reduce}\) per completed chunk. Combining this with the \(\bigO(1)\) amortized carry cost gives \(\bigO(\log s)\) calls per chunk.
    Since each chunk contains \(K\) tokens, the amortized number of calls per token is
    \[
        \bigO(\log s/K)=O(\log(t/K)/K).
    \]
    That proves the claim.
\end{proof}
\newpage
\section{Experimental details} \label{apx:experiments}

This appendix describes the implementation of our experiments in \Cref{sec:experiments}.
Detailed instructions for reproducing our experiments are provided with the code at \href{https://github.com/lukashaverbeck/kv-risk}{https://github.com/lukashaverbeck/kv-risk}.

\paragraph{Dataset and model.}
We evaluate on the ``long'' subset of LongBench-v2 \citep{bai2025longbench}.
The dataset is available at \url{https://longbench2.github.io/}.
We use the official LongBench-v2 evaluation protocol.
We run all experiments on Qwen3-32B \citep{qwen3}, using \texttt{int4} quantization.
We extend the context window to \(131\mathrm{k}\) tokens using YaRN RoPE scaling \citep{peng2024yarn}.
The model is available at \url{https://huggingface.co/Qwen/Qwen3-32B}.
We disable chain-of-thought generation and use greedy decoding throughout.
LongBench-v2 and Qwen3-32B are released under the Apache-2.0 license.

\paragraph{Computational resources.}
All experiments were conducted on a single NVIDIA H100 GPU with \(95{,}830\) MiB VRAM.
The full set of runs took approximately 100 GPU hours, including failed runs.

\paragraph{Implementation details.}
For the full-attention baseline, we use the standard implementation of the Qwen3-32B model.
For the compressed methods, we use the algorithms described in \Cref{sec:algorithms}, with two additions.
First, each token may attend to the first \(8\) tokens, following the attention-sink heuristic of \citet{xiao2023efficient}.
Second, each token may attend to the \(K\) tokens immediately preceding its chunk boundary.
Thus, each token attends to at most \(3K+8\) tokens.
This is compatible with our theory and enjoys the same guarantees.
We use \(K=2048\), corresponding to a compression rate of approximately \(95\%\) relative to the \(131\mathrm{k}\)-token context available to the full-attention baseline.

Many LongBench-v2 contexts exceed even the extended context window of Qwen3-32B.
For such examples, we follow the official truncation protocol of \citet{bai2025longbench} by removing the middle section of the prompt.

\paragraph{Local reducer instantiations.}
We report results for two instantiations of \(\reduce\): random sampling and protected clustering.
All randomized methods are evaluated with a single run using a fixed random seed.

For random sampling, we sample i.i.d.\ from the context measure, i.e., proportionally to
\begin{equation} \label{eq:experiments:weights}
    \textstyle
    \exp\!\left(\|k\|_2^2/(2\sqrt{d_k})\right),
\end{equation}
as described in \Cref{sec:setting}.

Protected clustering corresponds to \Cref{alg:cluster-reduce}, using the static feature map described in \Cref{apx:feature-map} with the value norm bound fixed to \(V^2=50\).
We evaluate protected ranks \(d=64\) and \(d=128\).
To compute the leading eigenspace, we use a randomized Nyström approximation with four-fold oversampling.

For \textsc{ClusterSlots} in \Cref{alg:cluster-reduce}, we use a sliced equal-mass clustering rule.
After the heavy atoms have been retained exactly as in \Cref{alg:cluster-reduce}, the light atoms \(i\in L\) have normalized masses \(p_i=q_i/\rho\) satisfying \(\sum_{i\in L}p_i=m\).
For each direction in a fixed bank of random directions in the protected coordinate space, we sort the light atoms by their one-dimensional projection and place intervals of length \(p_i\) consecutively on \([0,m]\).
The \(m\) latent slots are the unit intervals of this line, and the assignment matrix \(X\) is given by the overlap of each atom interval with each unit slot.
This construction satisfies the row and column constraints required in \Cref{alg:cluster-reduce}.
We choose the direction with the smallest within-slot squared error in the protected coordinates and sample one representative from each slot by inverse-transform sampling.

For this method, we treat the initial context measure as uniform over the tokens rather than proportional to \Cref{eq:experiments:weights}, as this significantly improved numerical stability in the computation of the leading eigenspace.

\paragraph{Literature baselines.} We compare performance against ScissorHands \citep{liu2023scissorhands}, SnapKV \citep{li2024snapkv}, and StreamingLLM \citep{xiao2023efficient}.
All three methods use a sliding window of the most recent \(4096\) tokens.
The allocation of the remaining \(2056\) token budget is dictated by the respective method. 
StreamingLLM retains the initial \(8\) sink tokens and uses the rest of this budget to extend the recent-token window.
SnapKV retains the top-ranked non-recent prompt tokens by attention score from the final prompt window, using average pooling with kernel size 5, as in the original paper \citep{li2024snapkv}.
ScissorHands retains the top-ranked non-recent tokens by cumulative pivotal-token counts.


\end{document}